\documentclass[accepted]{uai2021} 

\usepackage[american]{babel}

\usepackage{natbib} 
\usepackage{mathtools} 
\usepackage{booktabs} 
\usepackage{tikz} 

\usepackage{mathtools, amsmath, amsfonts, amsthm, subcaption}
\usepackage{thmtools} 
\usepackage{thm-restate}
\newcommand{\X}{X}
\newcommand{\y}{\mathbf{y}}
\newcommand{\f}{\mathbf{f}}
\newcommand{\m}{\mathbf{m}}
\renewcommand{\u}{\mathbf{u}}
\renewcommand{\S}{\mathbf{S}}
\newcommand{\reals}{\mathbb{R}}
\newcommand{\E}{\mathbb{E}}
\newcommand{\eqc}{\overset{c}{=}}

\newtheorem{rem}{Remark}

\title{The Promises and Pitfalls of Deep Kernel Learning}

\author[1]{\href{mailto:<swo25@cam.ac.uk>}{Sebastian~W.~Ober}{}} 
\author[1,2]{\href{mailto:<cer54@cam.ac.uk>}{Carl~E.~Rasmussen}{}}
\author[3]{\href{mailto:<m.vdwilk@imperial.ac.uk>}{Mark~van~der~Wilk}{}}
\affil[1]{%
        Department of Engineering\\
        University of Cambridge\\
        Cambridge, United Kingdom
}
\affil[2]{%
    Secondmind.ai\\
    Cambridge, United Kingdom
}
\affil[3]{%
    Department of Computing\\
    Imperial College London\\
    London, United Kingdom
}

\begin{document}
\maketitle

\begin{abstract}
Deep kernel learning (DKL) and related techniques aim to combine the representational power of neural networks with the reliable uncertainty estimates of Gaussian processes. One crucial aspect of these models is an expectation that, because they are treated as Gaussian process models optimized using the marginal likelihood, they are protected from overfitting. However, we identify situations where this is not the case. We explore this behavior, explain its origins and consider how it applies to real datasets. Through careful experimentation on the UCI, CIFAR-10, and the UTKFace datasets, we find that the overfitting from overparameterized maximum marginal likelihood, in which the model is ``somewhat Bayesian'', can in certain scenarios be worse than that from not being Bayesian at all. We explain how and when DKL can still be successful by investigating optimization dynamics. We also find that failures of DKL can be rectified by a fully Bayesian treatment, which leads to the desired performance improvements over standard neural networks and Gaussian processes.
\end{abstract}

\section{Introduction}
\label{introduction}
\begin{figure*}
    \centering
    \begin{subfigure}[b]{0.245\textwidth}
    \centering
    \centerline{\includegraphics[width=\textwidth]{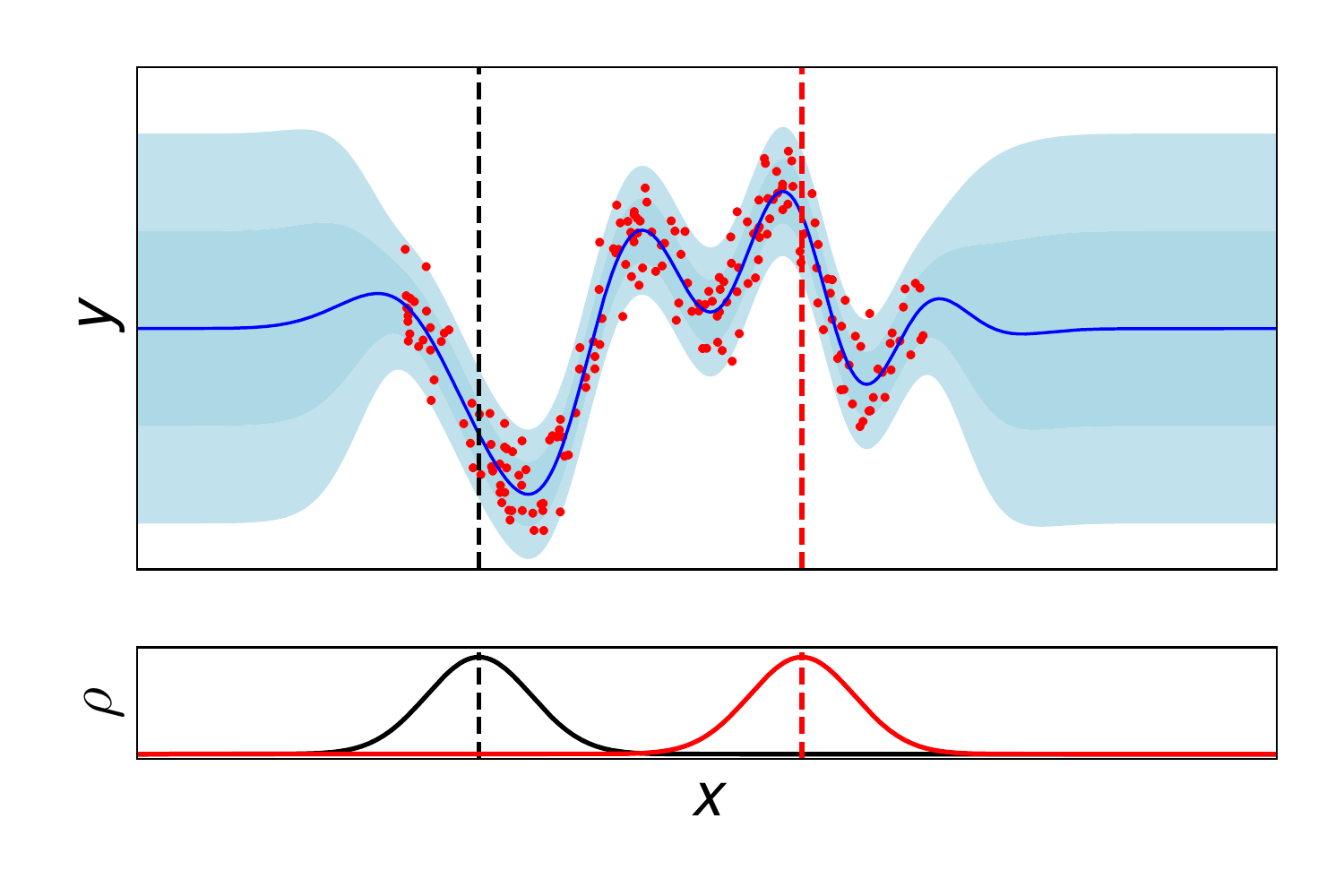}}
    \caption{SE kernel}
    \label{fig:toy_SE_fit}
    
    \end{subfigure}
    \hfill
    \begin{subfigure}[b]{0.245\textwidth}
    \centering
    \centerline{\includegraphics[width=\textwidth]{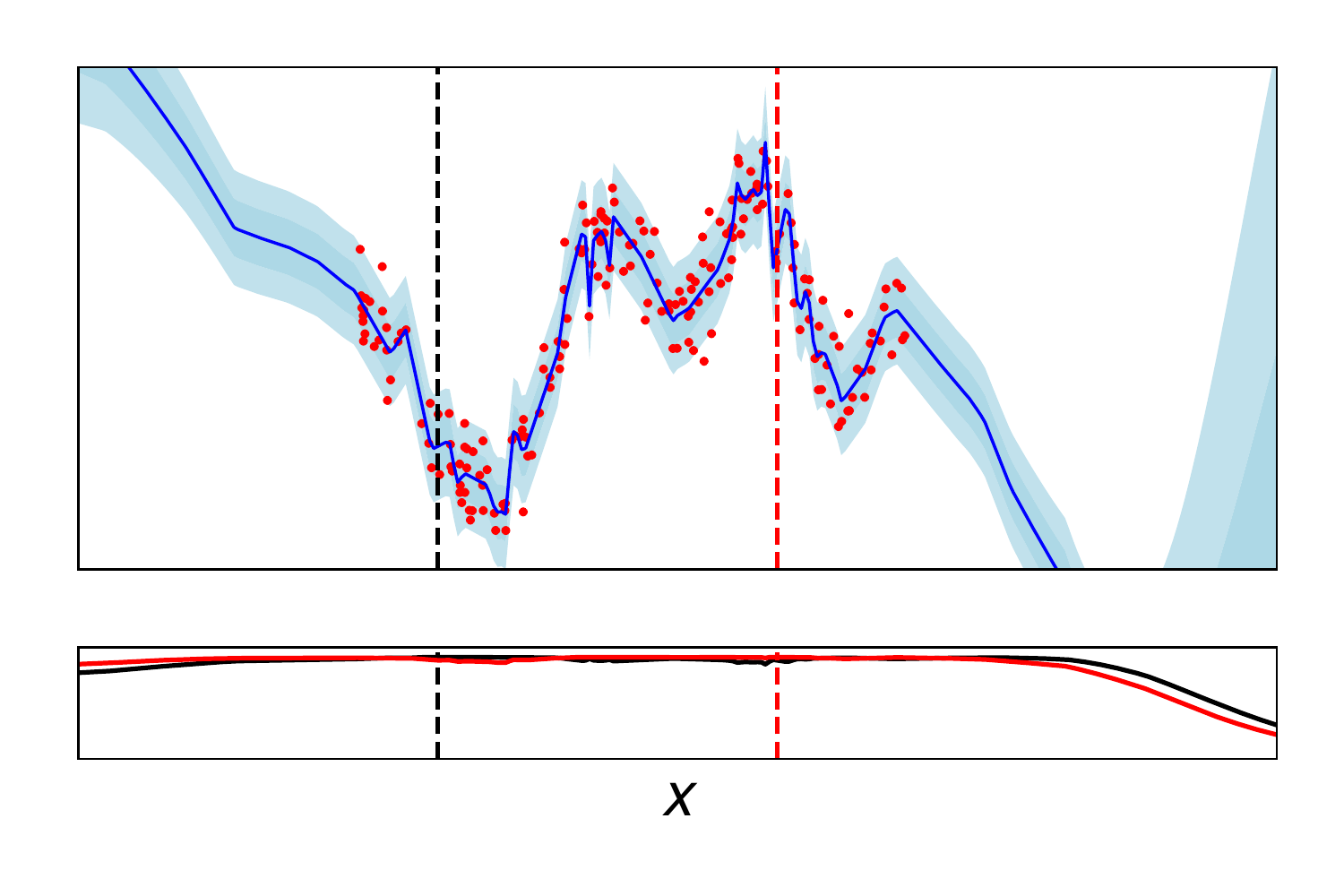}}
    \caption{Exact DKL kernel}
    \label{fig:toy_DKL_fit}
    
    \end{subfigure}
    \hfill
    \begin{subfigure}[b]{0.245\textwidth}
    \centering
    \centerline{\includegraphics[width=\textwidth]{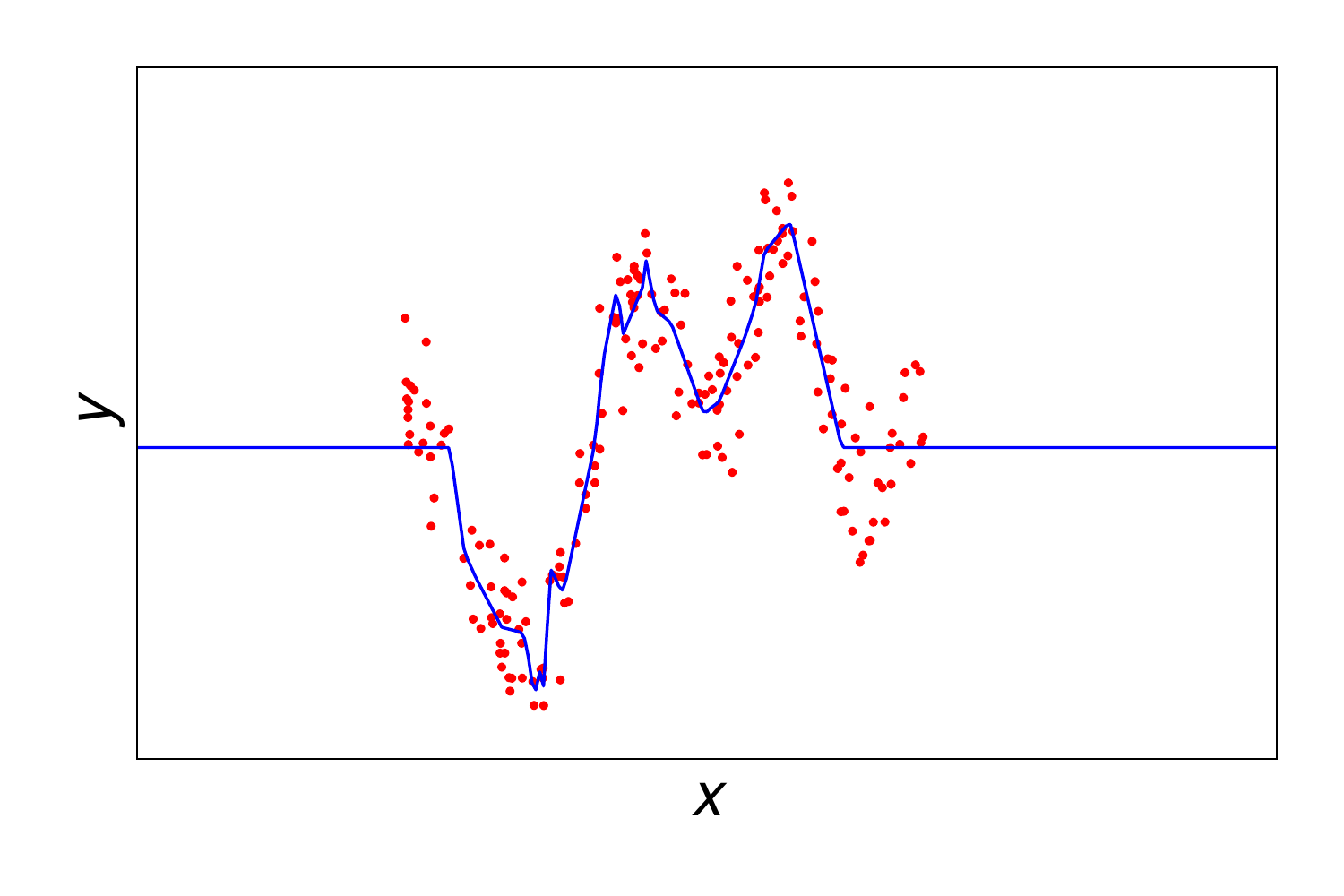}}
    \caption{Neural network fit}
    \label{fig:toy_NN_fit}
    
    \end{subfigure}
    \hfill
    \begin{subfigure}[b]{0.245\textwidth}
    \centering
    \centerline{\includegraphics[width=\textwidth]{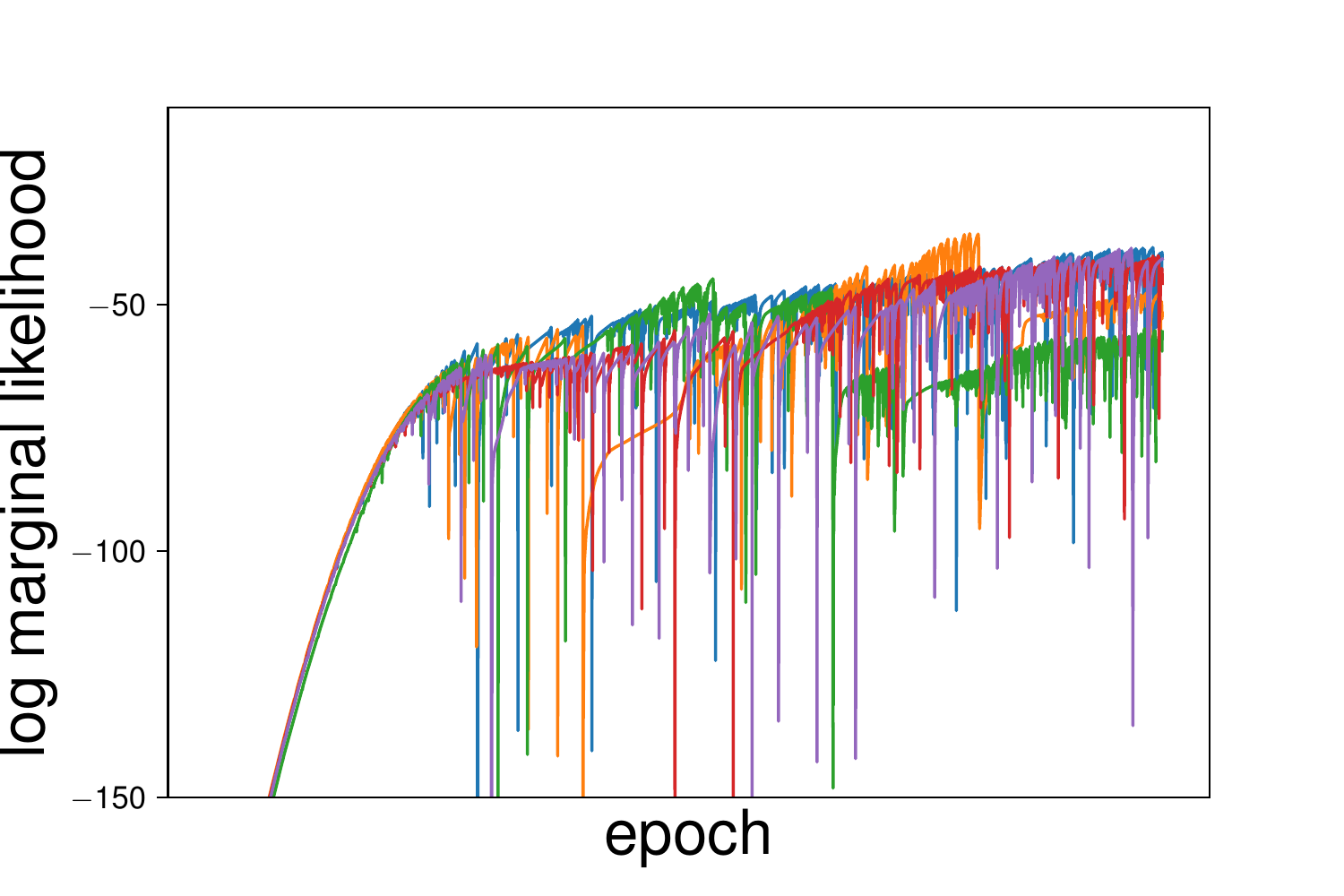}}
    \caption{LML training curves}
    \label{fig:toy_DKL_traincurves}
    
    \end{subfigure}
    \caption{Results on toy 1D dataset. Plots (a) and (b) show the predictive posterior for squared exponential (SE) and deep kernel learning (DKL) kernels, respectively; below each plot we also plot correlation functions $\rho_{x'}(x) = k(x, x')/\sigma_f^2$ at two points $x'$ given by the vertical dashed lines. (c) shows the fit given by the neural network analogous to the DKL model. Finally, (d) shows training curves of the log marginal likelihood (LML) for 5 different initializations of DKL.}
    \label{fig:toy_problem}

\end{figure*}

Gaussian process (GP) models \citep{rasmussen2006gaussian} are popular choices for Bayesian modeling due to their interpretable nature and reliable uncertainty estimates. 
These models typically involve only a handful of kernel hyperparameters, which are optimized with respect to the marginal likelihood in an empirical Bayes, or type-II maximum likelihood, approach \citep{berger1985marglik,rasmussen2006gaussian,murphy2012machine}.
However, most popular kernels can only adjust a degree of smoothing, rather than learn sophisticated representations from the data that might aid predictions.
This greatly limits the applicability of GPs to high-dimensional and structured data such as images.\looseness=-1

Deep neural networks \citep{lecun2015deep}, on the other hand, are known to learn powerful representations which are then used to make predictions on unseen test inputs. 
While deterministic neural networks have achieved state-of-the-art performance throughout supervised learning and beyond, they suffer from overconfident predictions \citep{guo2017calibration}, and do not provide reliable uncertainty estimates.
The Bayesian treatment of neural networks attempts to address these issues; however, despite recent advances in variational inference and sampling methods for Bayesian neural networks (BNNs; e.g. \citet{dusenberry2020efficient, zhang2019cyclical}), inference in BNNs remains difficult due to complex underlying posteriors and the large number of parameters in modern BNNs.
Moreover, BNNs generally require multiple forward passes to obtain multiple samples of the predictive posterior to average over.\looseness=-1

It is natural, therefore, to try to combine the uncertainty-representation advantages of GPs with the representation-learning advantages of neural networks, and thus obtain the ``best of both worlds.'' 
Ideally, such an approach would achieve the desiderata of a Bayesian model: training without overfitting, good uncertainty representation, and the ability to learn hyperparameters without using a validation set.
In this paper, we focus on a line of work that tries to achieve these called \textit{deep kernel learning} (DKL) \citep{calandra2016manifold, wilson2016deep, wilson2016stochastic}.
These works use a neural network to map inputs to points in an intermediate feature space, which is then used as the input space for a GP.
The network parameters can be treated as hyperparameters of the kernel, and thus are optimized with respect to the (log) marginal likelihood, as in standard GP inference.
This leads to an end-to-end training scheme that results in a model that hopefully benefits from the representational power of neural networks while also enjoying the benefits of reliable uncertainty estimation from the GP.
Moreover, as the feature extraction done by the neural network is deterministic, inference only requires one forward pass of the neural net, unlike fully Bayesian BNNs.
Previous works have shown that these methods can be used successfully \citep{calandra2016manifold, wilson2016deep, wilson2016stochastic, bradshaw2017adversarial}.\looseness=-1

We investigate to what extent DKL is actually able to achieve flexibility and good uncertainty, and what makes it successful in practice: for DKL to be useful from a Bayesian perspective, a higher marginal likelihood should lead to better performance.
In particular, it is often claimed that optimizing the marginal likelihood will automatically calibrate the complexity of the model, preventing overfitting.
For instance, \citet{wilson2016deep} states ``the information capacity of our model grows with the amount of available data, but its complexity is automatically calibrated through the marginal likelihood of the Gaussian process, without the need for regularization or cross-validation.''
This claim is based on the common decomposition of the log marginal likelihood into ``data fit'' and ``complexity penalty'' terms \citep{rasmussen2006gaussian}, which leads to the belief that a better marginal likelihood will result in better test performance.\looseness=-1

This is generally true when selecting a small number of hyperparameters. However, in models like DKL with many hyperparameters, we show that marginal likelihood training can encourage overfitting that is \emph{worse} than that from a standard, deterministic neural network.
This is because the marginal likelihood tries to correlate \emph{all} the datapoints, rather than just those for which correlations will be important.
As most standard GP models typically only have a few hyperparameters, this sort of overfitting is not usually an issue, but when many hyperparameters are involved, as in DKL, they can give the model the flexibility to overfit in this way.
As such, our work has implications for all GP methods which use highly parameterized kernels, as well as methods that optimize more than a handful of model parameters according to the marginal likelihood or ELBO.\looseness=-1

In this work, we make the following claims:
\begin{itemize}
\itemsep0em
    \item Using the marginal likelihood can lead to overfitting for DKL models.
    \item This overfitting can actually be worse than the overfitting observed using standard maximum likelihood approaches for neural networks.
    \item The marginal likelihood overfits by overcorrelating the datapoints, as it tries to correlate all the data, not just the points that should be correlated.
    \item Stochastic minibatching can mitigate this overfitting, and helps DKL to work in practice.
    \item A fully Bayesian treatment of deep kernel learning can avoid overfitting and obtain the benefits of both neural networks and Gaussian processes.
\end{itemize}
We note that some works have discussed that overfitting can be an issue for Gaussian processes trained with the marginal likelihood \citep{rasmussen2006gaussian,cawley2010over,lalchand2020approximate}, and \citet{calandra2016manifold} mentions that overfitting can be an issue for their DKL model. 
We additionally explain the undesirable behavior that DKL methods can exhibit, and the mechanism with which the marginal likelihood overfits.\looseness=-1

\section{Related Work}
\label{sec:related}
Full Bayesian inference in deep models can provide useful uncertainty estimates \citep[e.g.][]{blundell2015bbb,dusenberry2020efficient,osawa2019practical} and reduce both overfitting and the need to tune hyperparameters on a validation set \citep[e.g.][]{salimbeni2017doubly,ober2020global,immer2021}. However, these methods are more costly than plain DNNs. Interest remains in models that perform interest only in the final layer, as these provide uncertainty estimates in a single forward pass.

\citet{salakhutdinov2007using} first used deep belief networks to pretrain a neural network feature extractor to transform the inputs to a GP, followed by fine-tuning using the marginal likelihood.
\citet{calandra2016manifold} removed the deep belief network pretraining and only used the marginal likelihood to train the model.
\citet{wilson2016deep} improved the scalability of this model by using KISS-GP \citep{wilson2015kernel}, referring to the result as ``deep kernel learning''.
This was further extended to non-regression likelihoods and multiple outputs in \citet{wilson2016stochastic} by using stochastic variational inference \citep{hensman2015scalable}, resulting in stochastic variational deep kernel learning (SVDKL).
One of the most popular models in recent years has been the ``neural linear'' model \citep{riquelme2018deep, ober2019benchmarking}, which can be viewed as DKL with a linear kernel, or equivalently, Bayesian inference over the last layer of a neural network.\looseness=-1
These approaches, which use the marginal likelihood to optimize the neural network parameters, have been shown to be advantageous in multiple situations, including transfer testing and adversarial robustness \citep{bradshaw2017adversarial}.

However, \citet{tran2019calibrating} showed that these models can be poorly calibrated, and proposed Monte Carlo dropout \citep{gal2016dropout} to perform approximate Bayesian inference over the neural network weights in the model to fix this.
In addition, \citet{ober2019benchmarking} showed that it is difficult to get the neural linear model to perform well for regression without considerable hyperparameter tuning, and that fully Bayesian approaches for BNNs often require much less tuning to obtain comparable results.
Recent approaches (e.g. \citet{liu2020simple, van2021improving}) carefully regularize the neural network to mitigate these issues, but do still require tuning some hyperparameters on a validation set.\looseness=-1

\section{Background}
\label{sec:background}
\subsection{Gaussian Processes}

A Gaussian process (GP) can be seen as a distribution on functions with the property that every finite set of values is distributed according to a multivariate normal. 
In regression, we have pairs of inputs $\X = (x_1, \dots, x_N)^T$, $x_n \in \reals^D$ and outputs $\y = (y_1, \dots, y_N)^T$, $y_n \in \reals$, which we model as\looseness=-1
\begin{equation}
    \label{eq:gp-regression}
    y_n = f(x_n) + \epsilon_n, \; \epsilon_n \sim \mathcal{N}(0, \sigma_n^2),
\end{equation}
where $f$ has a GP prior, $f \sim \mathcal{GP}(m, k)$. 
The mean function $m:\reals^D\rightarrow\reals$ and positive semi-definite covariance function (or kernel) $k:\reals^D\times\reals^D\rightarrow\reals$ define the GP prior. We use a zero mean function throughout. The function values at a set of locations $\X$ are distributed according to $\mathcal{N}(0, K)$, where we define the kernel matrix $K \coloneqq K(\X, \X)$ to have $K_{ij} = k(x_i, x_j)$. 
Predictions of the latent function for a collection of test points $X_*$ can be computed in closed form:\looseness=-1
\begin{align}
    &\f_*|X, \y, X_* \sim \mathcal{N}(\mu_*, \Sigma_*), \textrm{ where} \\
    \nonumber \mu_* &= K(X_*, X)(K + \sigma_n^2 I_N)^{-1}\y, \\  
    \nonumber \Sigma_* &= K(X_*, X_*) - K(X_*, X)(K + \sigma_n^2 I_N)^{-1}K(X, X_*).
\end{align}
Finally, it is typical for the kernel to have a number of hyperparameters which are learned along with the noise variance $\sigma_n^2$ by maximizing the (log) marginal likelihood (LML, or model evidence), in an empirical Bayes, or type-II maximum likelihood approach:\looseness=-1
\begin{align}
    \log p(\y) &= \log \mathcal{N}(\y | \mathbf{0}, K + \sigma_n^2 I_N) \label{eq:compdf} \\
    \nonumber &\eqc -\underbrace{\frac{1}{2}\log |K + \sigma_n^2 I_N|}_{\text{(a) complexity}} - \underbrace{\frac{1}{2} \y^T(K + \sigma_n^2 I_N)^{-1}\y}_\text{(b) data fit},
\end{align}
We note that (a) and (b) are often referred to as the ``complexity penalty'' and ``data fit'' terms, respectively \citep{rasmussen2006gaussian}. 
For the purposes of this work, we use the automatic relevance determination (ARD) squared-exponential (SE) kernel, $k(x, x') = \sigma_f^2 \exp(-\frac{1}{2}\sum_{d=1}^D (x_d - x_d')^2/l_d^2)$.
Therefore, the hyperparameters to tune are the noise variance, $\sigma_n^2$, signal variance, $\sigma_f^2$, and lengthscales $l_d^2$.\looseness=-1

\subsection{Deep Kernel Learning}
One of the central critiques of GP regression is that it does not actually learn representations of the data. 
In an attempt to address this, several works \citep{calandra2016manifold, wilson2016deep, wilson2016stochastic, bradshaw2017adversarial} have proposed variants of deep kernel learning (DKL), which maps the inputs $x_n$ to intermediate values $v_n \in \reals^Q$ through a neural network $g_\phi(\cdot)$ parameterized by weights and biases $\phi$. 
These intermediate values are then used as inputs to the standard kernel resulting in the effective kernel $k_{DKL}(x, x') = k(g_\phi(x), g_\phi(x'))$. 
In order to learn the network weights and learn representations of the data, it was proposed to maximize the marginal likelihood with respect to the weights $\phi$ along with the kernel hyperparameters. 
We denote all the hyperparameters by $\theta \coloneqq \{\phi, \sigma_n, \sigma_f, \{l_q\}_{q=1}^Q\}$.\looseness=-1

Straightforward DKL suffers from two major drawbacks. 
First, the $\mathcal{O}(N^3)$ computational cost of GPs causes poor scalability in the number of data.\footnote{We note that \citet{wilson2016deep}, which first used the name ``DKL'', improved scalability using KISS-GP \citep{wilson2015kernel}; however, we use ``DKL'' when exact GP inference is used.}
Second, exact inference is only possible for Gaussian likelihoods, and therefore approximate techniques must be used for classification.
To achieve both, we follow \citet{bradshaw2017adversarial} in using stochastic variational inference (SVI) for GPs \citep{hensman2015scalable}, to result in stochastic variational DKL (SVDKL).\footnote{We note again that this is slightly different in exact implementation to the SVDKL model proposed in \citet{wilson2016stochastic}.}\looseness=-1

Considering the case of $C$ multiple outputs, we first introduce $M$ latent inducing variables $\u_c = (u_{c1}, \dots, u_{cM})^T$, indexed by $M$ inducing inputs $z_m \in \reals^Q$, which lie in the feature space at the output of the neural network.
We assume the standard variational posterior over the inducing variables, $q(\u_c) = \mathcal{N}(\m_c, \S_c)$, leading to an approximate posterior $q(\f, \u) = p(\f|\u)q(\u)$. 
We optimize the variational parameters $\m_c$ and $\S_c$, along with the model hyperparameters $\theta$, jointly by maximizing the evidence lower bound (ELBO):\looseness=-1
\begin{align}
\label{eq:ELBO}
    \mathcal{L} = \E_{q(\u)p(\f|\u)}[\log p(\y|\f)] - \mathrm{D}_{KL}(q(\u)||p(\u)).
\end{align}
Note that there are no restrictions on the likelihood $p(\y|\f)$ as the first term can be estimated using Monte Carlo sampling with the reparameterization trick \citep{kingma2013auto, rezende2014stochastic}. For Gaussian likelihoods and bounded inputs, theoretical results show that the ELBO can be made ``tight'' enough so it can be used as a stand-in for the marginal likelihood for hyperparameter optimization \citep{burt2020convergence}, if enough inducing points are given. Empirically, this has been shown to be the case for non-Gaussian likelihoods as well \citep{hensman2015scalable}.\looseness=-1

\section{Behavior in a Toy Problem}
\label{sec:toy}
\begin{figure}[t]
    \centering
    \begin{subfigure}[b]{0.23\textwidth}
    \centering
    \centerline{\includegraphics[width=\textwidth]{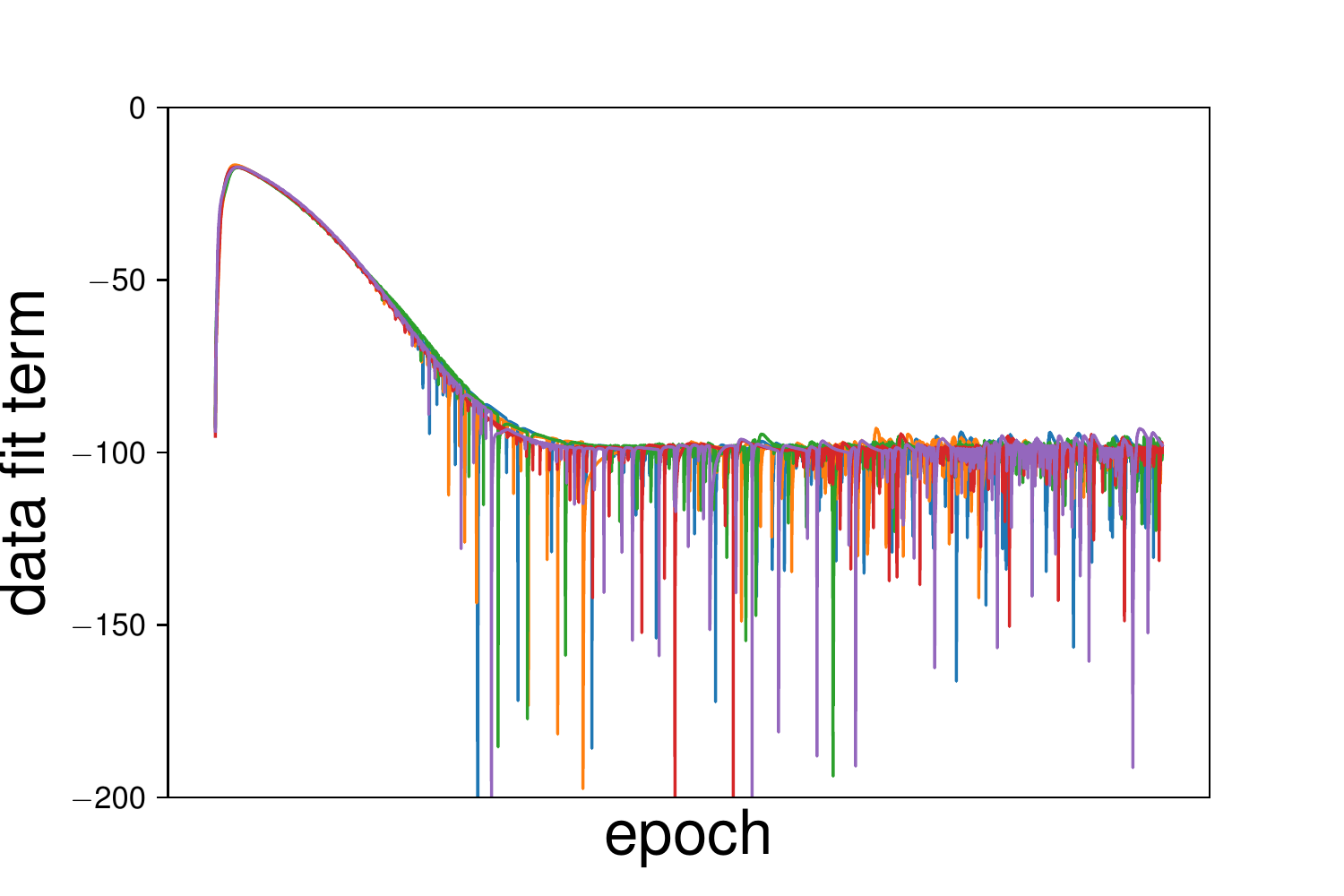}}
    \caption{Data fit}
    \label{fig:toy_fit_term}
    
    \end{subfigure}
    \hfill
    \begin{subfigure}[b]{0.23\textwidth}
    \centering
    \centerline{\includegraphics[width=\textwidth]{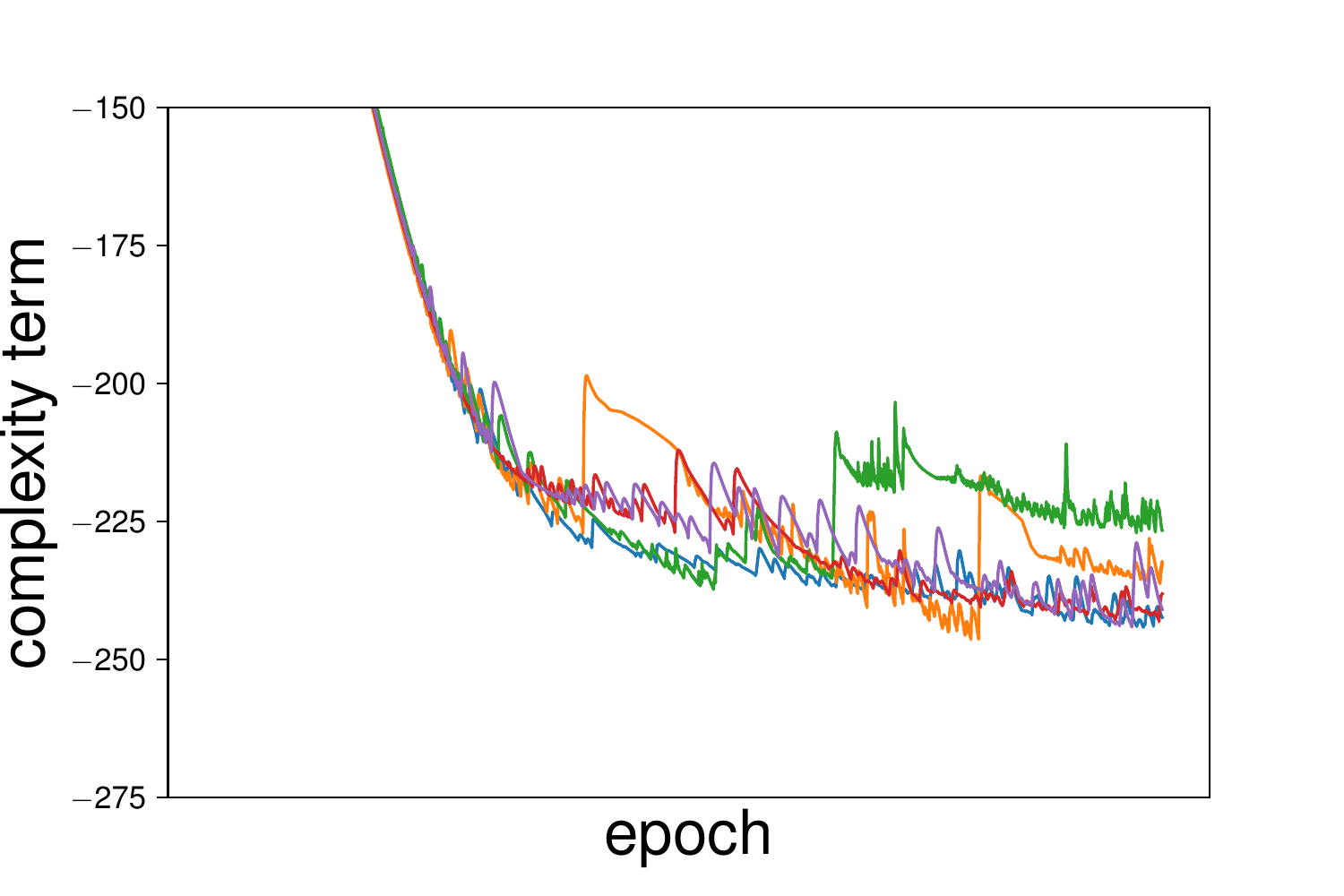}}
    \caption{Complexity penalty}
    \label{fig:toy_comp_term}
    
    \end{subfigure}
    \caption{Training curves for the data fit and complexity penalties of the log marginal likelihood for the toy problem.}
    \label{fig:f+c_curves}
\end{figure}
To motivate the rest of the paper, we first consider (exact) DKL a toy 1D regression problem \citep{snelson2006sparse} with 200 datapoints.
We consider DKL using a two hidden-layer fully-connected ReLU network with layer widths $[100, 50]$ as the feature extractor, letting $Q = 2$ with a squared exponential kernel for the GP.\footnote{We note that this is a smaller feature extractor than that proposed for a dataset of this size in \citet{wilson2016deep}.}
We describe the architecture and experimental details in more detail in App.~B.\looseness=-1
\begin{figure}
    \centering
    \includegraphics[width=0.5\textwidth]{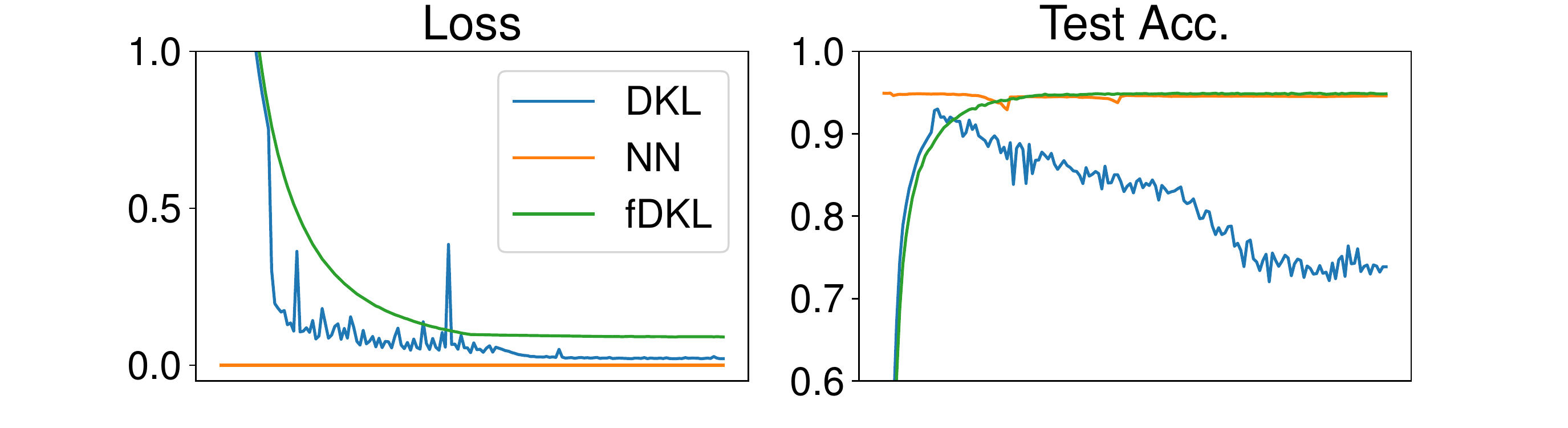}
    \caption{5k subset of MNIST, using a pretrained NN.}
    \label{fig:mnist}
\end{figure}

We plot the predictive posteriors of both a baseline GP with an SE kernel (corresponding to the ground truth), and DKL in Figures~\ref{fig:toy_SE_fit} and \ref{fig:toy_DKL_fit}, respectively.
We observe that DKL suffers from poor behavior: the fit is very jagged and extrapolates wildly outside the training data.
On the other hand, the fit given by the SE kernel is smooth and fits the data well without any signs of overfitting.
We therefore make the following observation:\looseness=-1
\begin{rem}\label{rem:overfitting-exists}
DKL models can be susceptible to overfitting, suggesting that the ``complexity penalty'' of the marginal likelihood may not always prevent overfitting.
\end{rem}\looseness=-1

We next compare to the fit given by the deterministic neural network which uses the same feature extractor as the DKL model, so that both models have the same depth.
To ensure a fair comparison, we retain the same training procedure, learning rates, full batch training, and number of optimization steps, so that we only change the model and training loss (from the LML to mean squared error).
We display the fit in Fig.~\ref{fig:toy_NN_fit}, which shows a nicer fit than the DKL fit of Fig.~\ref{fig:toy_DKL_fit}: while there is some evidence of overfitting, it is less than that of DKL.
This leads us to our second observation:\looseness=-1
\begin{rem}\label{rem:worse-overfitting}
DKL can exhibit \emph{worse} overfitting than a standard neural network trained using maximum likelihood.
\end{rem}\looseness=-1

We next plot training curves from five different runs of DKL in Fig.~\ref{fig:toy_DKL_traincurves}.
From these, we observe that training is very unstable, with many significant spikes in the marginal likelihood objective.
While we found that reducing the learning rate does improve stability, but only slightly (App.~C.1).
We also observe that runs often ends up settling in a different locations with different final values of the log marginal likelihood. 
We plot different fits from different initializations in App.~C.1, showing that these different local minima give very different fits with different generalization properties.\looseness=-1

In general, this behavior is concerning: one would hope that adding a Bayesian layer to a deterministic network would improve performance, as introducing Bayesian principles is often touted as a method to reduce overfitting (e.g. \citet{osawa2019practical}). 
However, based off this toy problem performance seems to worsen with the addition of a Bayesian layer at the output.
As this finding is seemingly at conflict with most of the literature, which has found that DKL, or variations thereof, can be useful, we devote the rest of this work to understanding when and why this pathology arises, including for real datasets.\looseness=-1

\begin{figure*}
    \centering
    \includegraphics[width=\linewidth]{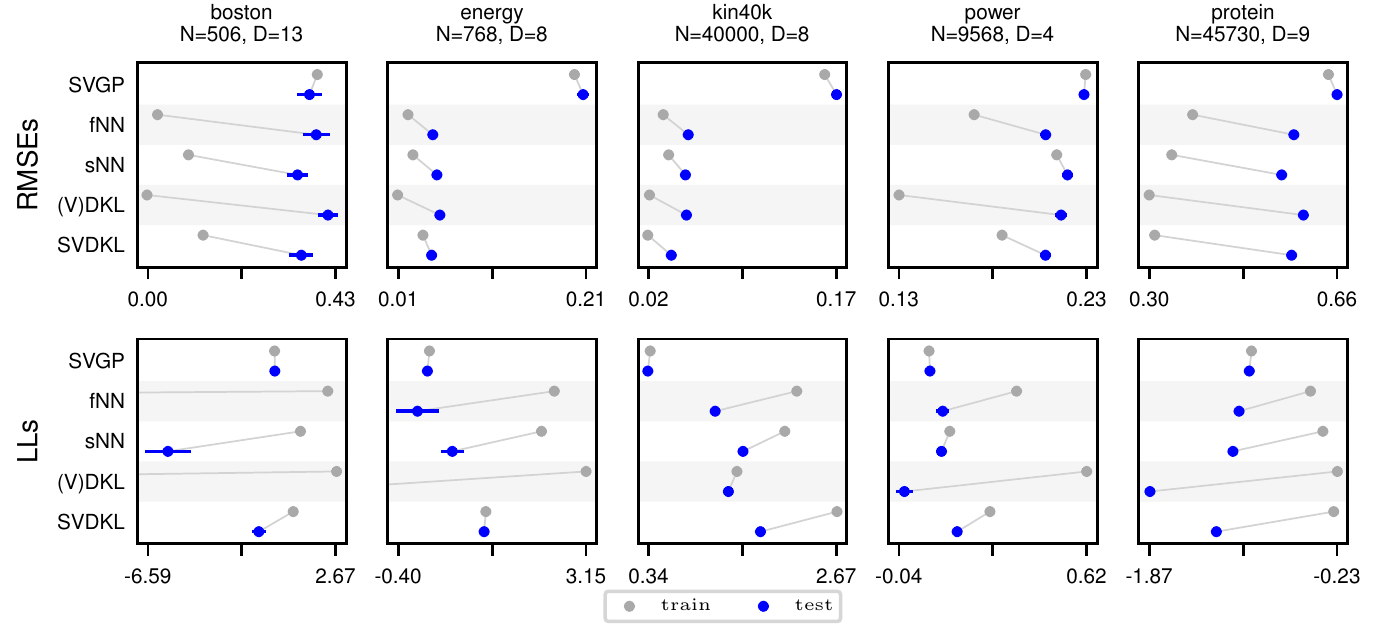}

    \caption{Results for the UCI datasets. We report train and test RMSEs and log likelihoods (LLs) for each method, averaged over the 20 splits. Left is better for RMSEs; right is better for LLs. Error bars represent one standard error.}
    \label{fig:uci_plots}

\end{figure*}

\section{Understanding the Pathology}
\label{sec:understanding}
\subsection{Regression}
To help understand the observed pathological behavior, we first look at the curves of the ``data fit'' and ``complexity penalties'' for five different initializations on the toy dataset. 
We present these curves in Fig.~\ref{fig:f+c_curves}. We note that each of the data fit curves largely stabilize around -100 nats, so that the complexity terms seem to account for most of the differences in the final marginal likelihood (Fig.~\ref{fig:toy_DKL_traincurves}). 
This behavior is explained by the following proposition, which states that the data fit term becomes uninteresting for any GPs with learnable signal variance trained on the marginal likelihood.\looseness=-1

\begin{restatable}[]{prop}{lmlprop}\label{prop:lml}
Consider the GP regression model as described in Eq.~\ref{eq:gp-regression}. 
Then, for any valid kernel function that can be written in the form $k(x, x') = \sigma_f^2 \hat{k}(x, x')$, where $\sigma_f^2$ is a learnable hyperparameter along with learnable noise $\sigma_n^2$ (and any other kernel hyperparameters), we have that the ``data fit'' term will equal $-N/2$ (where $N$ is the number of datapoints) at the optimum of the marginal likelihood.
\end{restatable}\looseness=-1

The proof (App.~A) is achieved by differentiation with respect to $\sigma_f$.
This result is far-reaching, as the use of a learnable signal variance $\sigma_f^2$ is almost universal.
This proposition therefore implies that after training the complexity penalty is responsible for any difference in marginal likelihood for GPs with different kernels.
Recall that the complexity penalty is given by\looseness=-1
\begin{equation}
    \label{eq:complexity}
    \frac{1}{2}\log |K + \sigma_n^2 I_N| = \frac{N}{2}\log \sigma_f^2 + \frac{1}{2}\log |\hat{K} + \hat{\sigma}_n^2 I_N|.
\end{equation}\looseness=-1
Maximizing the marginal likelihood encourages this term to be minimized, which can be done in at least two ways: minimizing $\sigma_f$, or minimizing the $|\hat{K} + \hat{\sigma}_n^2 I_N|$. 
However, there is little freedom in minimizing $\sigma_f$, because that would compromise the data fit. 
Therefore, the main mechanism for minimizing the complexity penalty would be through minimizing the second term. 
One way of doing this is to correlate the input points as much as possible: if there are enough degrees of freedom in the kernel, it is possible to ``hack'' the Gram matrix so that it can do this while minimizing the impact on the data fit term. 
We see this by looking at the correlation plots for the SE and DKL fits in Fig.~\ref{fig:toy_problem}: below the plots of the predictive posteriors, we have plotted correlation functions $\rho_{x'}(x) = k(x, x')/\sigma_f^2$ at two points $x'$ given by the vertical dashed lines. 
We see that, while Fig.~\ref{fig:toy_SE_fit} shows the expected Gaussian bump for the SE kernel, Fig.~\ref{fig:toy_DKL_fit} shows near-unity correlation functions for all values. 
Furthermore, in Appendix~C.1 we show empirically that for fits that do not show as much correlation, the final marginal likelihood is worse, suggesting that increasing the correlation is indeed the main mechanism by which the model increases its marginal likelihood.
We note that one mechanism of correlating all the datapoints has concurrently been explored and termed ``feature collapse'' \citep{van2021improving}, where the neural network feature extractor learns to collapse all the datapoints onto a low-dimensional surface.
We summarize our findings in the remark:\looseness=-1

\begin{rem}\label{rem:correlation}
The complexity penalty encourages high correlation between different points. 
Overparameterizing the  covariance function can lead to pathological results, as it allows all points to be correlated in the prior, not only the points where we would like correlations to appear.
\end{rem}\looseness=-1

\subsection{Classification}
We now briefly consider classification. We compare a neural network (NN) using the usual softmax loss with DKL. Due to the non-Gaussian likelihood, we use the variational approximation to the marginal likelihood (Eq.\ref{eq:ELBO}). We also compare a DKL model with the NN feature extractor fixed to what is obtained from the normal NN training procedure, which we refer to as `fixed net DKL' (fDKL).
All models are initialized with the same pretrained neural network using standard NN training for a fair comparison.
Fig.~\ref{fig:mnist} shows the training curves for the losses and test accuracies on a subset of 5000 points of MNIST. All models are trained with full batches (see App.~B for additional details). We observe that the standard NN has a near-zero loss without worsening test accuracy. DKL also attains low loss but significantly overfits. The loss obtained by fDKL is the highest of the three models, but does not overfit, and achieves the best test accuracy by a small margin.

We can explain the results in a similar way to regression, even though the SVDKL loss is different. In this case, the expected log-likelihood measures the data fit, while the KL enforces simplicity of the prior and approximate posterior.\footnote{Indeed, the KL term contains a $\frac{1}{2}\log|K|$ term, just like Eq.~\ref{eq:compdf}, only where the kernel is evaluated at the inducing points.}
Since the MNIST classes are well-separated, we expect a near-zero data fit term. In the standard NN training loss, there is little encouragement to overfit, since a well-fitted model already achieves a loss close to the global minimum of zero. The DKL objective, on the other hand, contains the complexity penalty which can be further reduced by over-correlating points, just as in regression.

We now investigate how these observations relate to real, complex datasets, as well as to the prior literature which has shown that DKL can obtain good results.\looseness=-1

\section{DKL for Real Datasets}
\label{sec:realdata}
Despite these findings, multiple works have shown that DKL methods can perform well in practice \citep{wilson2016stochastic, bradshaw2017adversarial}.
We now consider experiments on various datasets and architectures to further investigate the observed pathological behavior and how DKL succeeds.
We provide full experimental details in Appendix~B and additional experimental results in Appendix~C.\looseness=-1

\subsection{DKL for UCI Regression}
\begin{table}
\small
    \centering
    \caption{LMLs/ELBOs per datapoint for UCI datasets.}\label{tab:uci_elbos}
    \begin{tabular}{rccc}
\toprule 
& SVGP & (V)DKL & SVDKL \\ 
\midrule 
\textsc{Boston}& -1.66 $\pm$ 0.06 & $\mathbf{2.47 \pm 0.00}$ & 0.47 $\pm$ 0.01 \\ 
\textsc{Energy}& -0.07 $\pm$ 0.01 & $\mathbf{3.01 \pm 0.02}$ & 1.21 $\pm$ 0.00 \\ 
\textsc{Kin40K}& 0.14 $\pm$ 0.00 & 1.41 $\pm$ 0.00 & $\mathbf{2.62 \pm 0.00}$ \\ 
\textsc{Power}& 0.01 $\pm$ 0.00 & $\mathbf{0.57 \pm 0.00}$ & 0.25 $\pm$ 0.00 \\
\textsc{Protein}& -1.06 $\pm$ 0.00 & -$\mathbf{0.32 \pm 0.01}$ & -0.35 $\pm$ 0.00 \\ 
\bottomrule
\end{tabular}
\end{table}
\begin{table*}
\scriptsize
    \centering
    \caption{Results for the UTKFace age regression task and CIFAR-10 classification, without data augmentation. We report means plus/minus one standard error, averaged over three runs.}\label{tab:batch}
    \begin{tabular}{rcccccccc}
\toprule 
& \multicolumn{5}{c}{Batch size: 100} & \multicolumn{3}{c}{Batch size: 200/500} \\ 
\cmidrule(lr){2-6}
\cmidrule(lr){7-9}
& NN & SVDKL & pNN & fSVDKL & pSVDKL & pNN & fSVDKL & pSVDKL \\ 
\midrule
UTKFace - ELBO & - & 0.92 $\pm$ 0.01 & - & 1.05 $\pm$ 0.30 & 1.03 $\pm$ 0.10 & - & 0.75 $\pm$ 0.34 & 1.43$\pm$0.04 \\ 
Train RMSE & 0.04$\pm$0.00 & 0.04$\pm$0.00 & 0.04$\pm$0.00 & 0.08 $\pm$ 0.03 & 0.04$\pm$0.00 & 0.04 $\pm$ 0.00 & 0.12 $\pm$ 0.03 & 0.04 $\pm$ 0.00 \\ 
Test RMSE & 0.40 $\pm$ 0.00 & 0.40 $\pm$ 0.01 & 0.41 $\pm$ 0.00 & 0.31 $\pm$ 0.07 & 0.38 $\pm$ 0.02 & 0.39 $\pm$ 0.01 & 0.23 $\pm$ 0.07 & 0.34 $\pm$ 0.02 \\ 
Train LL & 1.81 $\pm$ 0.01 & 1.30 $\pm$ 0.01 & 1.83 $\pm$ 0.01 & 1.16 $\pm$ 0.31 & 1.20 $\pm$ 0.08 & 1.83 $\pm$ 0.01 & 0.82 $\pm$ 0.34 & 1.60 $\pm$ 0.03 \\ 
Test LL & -48.73 $\pm$ 1.64 & -6.88 $\pm$ 0.38 & -53.72 $\pm$ 1.71 & -7.55 $\pm$ 3.42 & -4.74 $\pm$ 1.35 & -48.48 $\pm$ 2.07 & -5.36 $\pm$ 4.78 & -10.43 $\pm$ 2.94 \\
\midrule
CIFAR-10 - ELBO & - & -0.76 $\pm$ 0.28 & - & -0.02 $\pm$ 0.00 & -0.00 $\pm$ 0.00 & - & -0.02 $\pm$ 0.00 & -0.00 $\pm$ 0.00 \\ 
Train Acc. & 1.00 $\pm$ 0.00 & 0.76 $\pm$ 0.09 & 1.00 $\pm$ 0.00 & 1.00 $\pm$ 0.00 & 1.00 $\pm$ 0.00 & 1.00 $\pm$ 0.00 & 1.00 $\pm$ 0.00 & 1.00 $\pm$ 0.00 \\ 
Test Acc. & 0.79 $\pm$ 0.00 & 0.63 $\pm$ 0.03 & 0.79 $\pm$ 0.00 & 0.78 $\pm$ 0.00 & 0.79 $\pm$ 0.00 & 0.79 $\pm$ 0.00 & 0.79 $\pm$ 0.00 & 0.79 $\pm$ 0.00 \\ 
Train LL & -0.00 $\pm$ 0.00 & -0.71 $\pm$ 0.28 & -0.00 $\pm$ 0.00 & -0.01 $\pm$ 0.00 & -0.00 $\pm$ 0.00 & -0.00 $\pm$ 0.00 & -0.00 $\pm$ 0.00 & -0.00 $\pm$ 0.00 \\ 
Test LL & -2.05 $\pm$ 0.03 & -1.37 $\pm$ 0.10 & -2.30 $\pm$ 0.11 & -1.14 $\pm$ 0.00 & -1.13 $\pm$ 0.01 & -2.88 $\pm$ 0.04 & -1.07 $\pm$ 0.01 & -1.45 $\pm$ 0.00 \\ 
Inc. Test LL & -8.87 $\pm$ 0.10 & -3.38 $\pm$ 0.77 & -9.48 $\pm$ 0.30 & -5.10 $\pm$ 0.01 & -5.24 $\pm$ 0.05 & -10.77 $\pm$ 0.07 & -4.73 $\pm$ 0.03 & -6.63 $\pm$ 0.03 \\ 
ECE & 0.18 $\pm$ 0.00 & 0.10 $\pm$ 0.05 & 0.19 $\pm$ 0.00 & 0.14 $\pm$ 0.00 & 0.15 $\pm$ 0.00 & 0.19 $\pm$ 0.00 & 0.13 $\pm$ 0.00 & 0.15 $\pm$ 0.00 \\ 
\bottomrule 
    \end{tabular}
\end{table*}

We first consider DKL applied to a selection of regression datasets from the UCI repository \citep{Dua:2019}: \textsc{Boston}, \textsc{Energy}, \textsc{Kin40K}, \textsc{Power}, \textsc{Protein}.
These represent a range of different sizes and dimensions: \textsc{Energy}, \textsc{Power}, and \textsc{Protein} were chosen specifically because we expect that they can benefit from the added depth to a GP \citep{salimbeni2017doubly}.\looseness=-1

We consider a range of different models, and we report train and test root mean square errors (RMSEs) and log likelihoods (LLs) in Fig.~\ref{fig:uci_plots}, and tabulate the log marginal likelihoods (LMLs) or ELBOs in Table~\ref{tab:uci_elbos}.
First, we consider a baseline stochastic variational GP (SVGP) model with an ARD SE kernel.
As this is a GP model with few hyperparameters, we would not expect significant differences between training and testing performances. 
Indeed, looking at Fig.~\ref{fig:uci_plots}, this is exactly what we observe: the test performance is comparable to, and sometimes even slightly better than, the training performance for both RMSEs and LLs.\looseness=-1

We compare to a neural network trained with mean squared error loss and DKL using the same neural network architecture for feature extractor (so that the depths are equal).
We first consider DKL models where we use full-batch training, compared to a neural network with full-batch training, which we refer to as fNN.
As full-batch training for DKL is expensive for larger datasets, for the \textsc{Kin40K}, \textsc{Power}, and \textsc{Protein} we instead use SVDKL trained with 1000 inducing points but full training batches, which we term \emph{variational} DKL (VDKL).
For both methods we use a small weight decay to help reduce overfitting, and we use the same number of gradient steps are used for each to ensure a fair comparison.
Looking at the results for fNN and (V)DKL in Fig.~\ref{fig:uci_plots}, we see that both of these methods overfit quite drastically.
This mirrors our observations in Remark~\ref{rem:overfitting-exists} that DKL models can be susceptible to overfitting.
In most cases the overfitting is noticeably worse for (V)DKL than it is for fNN, reflecting our observation in Remark~\ref{rem:worse-overfitting}.
This is particularly concerning for the log likelihoods, as one would hope that the ability of DKL to express epistemic uncertainty through the last-layer GP would give it a major advantage over the neural network, which cannot do so.\looseness=-1

In practice, however, many approaches for DKL and neural networks alike make use of stochastic minibatching during training.
In fact, it is well-known that minibatch training induces implicit regularization for neural networks that helps generalization \citep{keskar2016large}.
We therefore investigate this for both DKL and neural networks: we refer to the stochastic minibatched network as sNN and compare to SVDKL, using the same batch sizes for both.
Referring again to Fig.~\ref{fig:uci_plots}, we see that minibatching generally reduces overfitting compared to the full-batch versions, for both model types.
Moreover, the difference between the full batch and stochastic minibatch performances of DKL seem to be greater than the corresponding differences for the standard neural networks, suggesting that the implicit regularization effect is stronger.
The exception to this trend is \textsc{Kin40K}, which appears to be low-noise and simple for a deep model to predict for.
We also note that with the exception of protein, SVDKL now performs the best of the deep models in terms of log likelihoods, and generally performs better than SVGP.\looseness=-1

Finally, we consider Table~\ref{tab:uci_elbos}, which shows the ELBOs/LMLs for each of the GP methods.
SVGP has by far the worse ELBOs, whereas (V)DKL generally has by far the best.
It is important to note that the ELBOs for SVDKL are worse than those for (V)DKL despite its generally better test performance.
This suggests that improving the marginal likelihood for DKL models does not improve test performance, as one would desire for a Bayesian model.
We summarize our findings in the following remark:\looseness=-1
\begin{rem}\label{rem:minibatching}
The reason for DKL's successful performance is not an improved marginal likelihood, but rather that stochastic minibatching provides implicit regularization that protects against overfitting with the marginal likelihood.
\end{rem}\looseness=-1

Therefore, we observe again that the Bayesian benefits of the marginal likelihood do not apply in the overparameterized regime: indeed, we find that using the marginal likelihood can be worse than not being Bayesian at all.\looseness=-1

\subsection{DKL for Image Datasets}\label{sec:image-datasets}
\begin{table*}
\scriptsize
    \centering
    \caption{Results for the image datasets with data augmentation. We report means $\pm1$ standard error, averaged over 3 runs.}\label{tab:batch_da}
    \begin{tabular}{rcccccccc}
\toprule 
& \multicolumn{4}{c}{Batch size: 100} & \multicolumn{3}{c}{Batch size: 200 (UTKFace) / 500 (CIFAR-10)} \\ 
\cmidrule(lr){2-5}
\cmidrule(lr){6-8}
& NN & pNN & fSVDKL & pSVDKL & pNN & fSVDKL & pSVDKL \\ 
\midrule 
UTKFace - ELBO & - & - & 0.16 $\pm$ 0.03 & 0.14 $\pm$ 0.03 & - & 0.12 $\pm$ 0.06 & 0.45 $\pm$ 0.03 \\ 
Train RMSE & 0.19 $\pm$ 0.01 & 0.18 $\pm$ 0.00 & 0.19 $\pm$ 0.00 & 0.17 $\pm$ 0.01 & 0.13 $\pm$ 0.00 & 0.20 $\pm$ 0.01 & 0.12 $\pm$ 0.01 \\ 
Test RMSE & 0.36 $\pm$ 0.00 & 0.36 $\pm$ 0.00 & 0.36 $\pm$ 0.00 & 0.35 $\pm$ 0.00 & 0.35 $\pm$ 0.00 & 0.31 $\pm$ 0.04 & 0.35 $\pm$ 0.01 \\ 
Train LL & 0.25 $\pm$ 0.03 & 0.31 $\pm$ 0.01 & 0.25 $\pm$ 0.03 & 0.30 $\pm$ 0.03 & 0.65 $\pm$ 0.02 & 0.20 $\pm$ 0.06 & 0.63 $\pm$ 0.04 \\ 
Test LL & -1.03 $\pm$ 0.07 & -1.22 $\pm$ 0.05 & -0.92 $\pm$ 0.07 & -0.76 $\pm$ 0.03 & -2.72 $\pm$ 0.21 & -0.63 $\pm$ 0.30 & -1.55 $\pm$ 0.17 \\ 
\midrule
CIFAR-10 - ELBO & - & - & -0.07 $\pm$ 0.00 & -0.03 $\pm$ 0.00 & - & -0.06 $\pm$ 0.01 & -0.01 $\pm$ 0.00 \\ 
Train Acc. & 0.98 $\pm$ 0.00 & 0.99 $\pm$ 0.00 & 0.99 $\pm$ 0.00 & 0.99 $\pm$ 0.00 & 1.00 $\pm$ 0.00 & 0.98 $\pm$ 0.00 & 1.00 $\pm$ 0.00 \\ 
Test Acc. & 0.86 $\pm$ 0.00 & 0.86 $\pm$ 0.00 & 0.86 $\pm$ 0.00 & 0.86 $\pm$ 0.00 & 0.87 $\pm$ 0.00 & 0.86 $\pm$ 0.00 & 0.86 $\pm$ 0.00 \\ 
Train LL & -0.05 $\pm$ 0.00 & -0.02 $\pm$ 0.00 & -0.05 $\pm$ 0.00 & -0.03 $\pm$ 0.00 & -0.01 $\pm$ 0.00 & -0.05 $\pm$ 0.01 & -0.01 $\pm$ 0.00 \\ 
Test LL & -0.70 $\pm$ 0.01 & -0.90 $\pm$ 0.00 & -0.68 $\pm$ 0.00 & -0.64 $\pm$ 0.00 & -1.38 $\pm$ 0.03 & -0.67 $\pm$ 0.02 & -0.84 $\pm$ 0.00 \\ 
Inc. Test LL & -4.83 $\pm$ 0.12 & -6.31 $\pm$ 0.00 & -4.65 $\pm$ 0.00 & -4.58 $\pm$ 0.00 & -8.97 $\pm$ 0.07 & -4.66 $\pm$ 0.13 & -6.06 $\pm$ 0.01 \\ 
ECE & 0.09 $\pm$ 0.00 & 0.11 $\pm$ 0.00 & 0.09 $\pm$ 0.00 & 0.09 $\pm$ 0.00 & 0.12 $\pm$ 0.00 & 0.09 $\pm$ 0.00 & 0.11 $\pm$ 0.00 \\ 
\bottomrule 
    \end{tabular}
\end{table*}

We now explore how these findings relate to high-dimensional, highly structure image datasets.
We might expect that the benefits of DKL would be stronger for images than in the previous regression datasets, as the design of kernels for these high-dimensional spaces remains an open question despite numerous recent advances \citep{van2017convolutional, dutordoir2020bayesian}, and neural networks generally perform far better than kernel methods.\looseness=-1

We first consider a regression problem using image inputs: an age regression task using the UTKFace dataset \citep{zhifei2017cvpr}.
The dataset consists of 23,708 images of size $200\times200\times3$ containing aligned and cropped faces.
These images are annotated with age, gender and race, where we focus on predicting age. We consider models based on a ResNet-18 \citep{he2016deep}: we take the standard ResNet-18 with 10-dimensional output, to which we add a ReLU nonlinearity and then either a linear output layer or an ARD SE GP, corresponding to the baseline neural network and SVDKL, respectively.
We consider different feature widths $Q$ in App.~C.
This construction ensures that both models have the same depth, so that any improvement observed for either cannot be attributed to the fact that the models have different depths.
We consider the baseline neural network (NN) and SVDKL models.
Additionally, as both \citet{wilson2016stochastic} and \citet{bradshaw2017adversarial} use a pretraining and finetuning procedure for their models, we compare to this as well.
We take the trained baseline NNs, and first learn the variational parameters and GP hyperparameters, keeping the network fixed.
We refer to the result as the fixed net SVDKL (fSVDKL) model; we then train everything jointly for a number of epochs, resulting in the pretrained SVKDL (pSVDKL) model.
Finally, so that any improvement for f/pSVDKL is not just from additional gradient steps, we also train the neural networks for the same number of epochs, resulting in the pretrained NN (pNN) model.
We average all results over 3 independent runs using a batch size of 100, and we refer the reader to App.~B.3 for full experimental details.\looseness=-1

We report ELBOs, train and test RMSEs, and train and test log likelihoods for the normalized data in the top left portion of Table~\ref{tab:batch} (batch size 100).
We see that SVDKL, the method without pretraining, obtains lower ELBOs than either fSVDKL or pSVDKL, which obtain largely similar ELBOs.
We suspect that this is because of the difficulty in training large DKL models from scratch, as noted in \citet{bradshaw2017adversarial}; this is also consistent with our earlier observation that training can be very unstable.
We see that each method, except fSVDKL (with the fixed pretrained network), achieves similar train RMSE, but the test RMSEs are significantly worse, with fSVDKL obtaining the best.
Unsurprisingly, the NN models perform poorly in terms of LL, as they are unable to express epistemic uncertainty. 
However, we also observe that additional training of the NNs worsens both test RMSEs and LLs.
pSVDKL (where the network is allowed to change after pretraining) obtains the best test LL of all methods, as well as better test RMSE than the neural networks, showing that SVDKL can yield improvements consistent with the prior literature.
We note, however, that there is still a substantial gap between train and test performance, indicating overfitting in a way consistent with Remark~\ref{rem:overfitting-exists}.\looseness=-1

\subsubsection{Increasing the Batch Size}
From our UCI experiments, we hypothesized that implicit regularization from minibatch noise was key in obtaining good performance for SVDKL (Remark~\ref{rem:minibatching}). 
We therefore consider increasing the batch size from 100 to 200 for the pretrained methods, keeping the pretrained neural networks the same (Table~\ref{tab:batch}, top right).
We make a few key observations.
First, this leads to a significantly improved ELBO for pSVDKL, which ends up helping the test RMSE.
However, we see that instead of improving the test LL, it becomes significantly worse, whereas the train LL becomes better: clear evidence of overfitting.
Moreover, fSVDKL, where the network is kept fixed, now outperforms pSVDKL, which has a better ELBO.
Finally, we note that the behavior of pNN does not change significantly, in fact slightly improving with increased batch size: this suggests that the implicit regularization from minibatching is stronger for SVDKL than for standard NNs.
All of these observations are consistent with our findings surrounding Remark~\ref{rem:minibatching}, which argues that stochastic minibatching is crucial to the success of DKL methods, and a better marginal likelihood is associated with worse performance.\looseness=-1

\subsubsection{Image Classification}
Our theory in Section~\ref{sec:understanding} only applies directly to regression. 
As one of the main successes of current deep learning is in classification, it is therefore natural to wonder whether the trends we have observed also apply to classification tasks.
We consider CIFAR-10 \citep{krizhevsky2009learning}, a popular dataset of $32\times32\times3$ images belonging to one of 10 classes.
We again consider a modified ResNet-18 model, in which we have ensured that the depths remain the same between NN and DKL models.
We consider training the models with batch sizes of 100 and 500.
We look at ELBOs, accuracies, and LLs, as well as the LL for incorrectly classified test points, which can indicate overconfidence in predicting wrongly.
We also look at expected calibration error (ECE; \citet{guo2017calibration}), a popular metric evaluating model calibration; results are shown in the lower portion of Table~\ref{tab:batch}.
Here, we see that plain SVDKL struggles even more to fit well, indicating the importance of pretraining.
For the batch size 100 experiments, pSVDKL generally performs the best, reflecting the experience of \citet{wilson2016stochastic} and \citet{bradshaw2017adversarial}.
However, we again observe that increasing the batch size hurts pSVDKL, and fSVKDL outperforms it despite worse ELBOs.\looseness=-1

\subsection{Data Augmentation}
It is common practice with image datasets to perform data augmentation, which effectively increases the size of the training dataset\footnote{Bayesian inference does not permit this, instead requiring that the model be adjusted \citep{vdw2018inv,nabarro2021}. Very recently, \citet{schwoebel2021last} applied this to DKL.} by using modified versions of the images.
We briefly consider whether this changes the overfitting behavior we observed, by repeating the same experiments (without plain SVDKL, as it struggles to fit) with random cropping and horizontal flipping augmentations; see Table~\ref{tab:batch_da}.
Overall, we once again find that increasing the batch size still significantly hurts the performance of pSVDKL: whereas pSVDKL outperforms the fixed-network version for batch size 100, larger batch sizes reverse this, so that finetuning the network according to the ELBO hurts, rather than helps, performance.
Therefore, in this case, using last-layer Bayesian inference is worse than not being Bayesian at all.
These results reflect our findings in the previous remarks that using the marginal likelihood can be worse than using a standard likelihood, and that stochastic minibatching is one of the main reasons that DKL can be successful.\looseness=-1

\section{Addressing the Pathology}\label{sec:fix}
We have seen that the empirical Bayesian approach to overparameterized GP kernels can lead to pathological behavior.
In particular, we have shown that methods that rely on the marginal likelihood to optimize a large number of hyperparameters can overfit, and that learning is unstable.
While minibatching can help mitigate these issues, the overall performance is sensitive to the batch size, leading to a separate hyperparameter to tune.
It is therefore natural to wonder whether we can address this by using a fully Bayesian approach, which has been shown to improve the predictive uncertainty of GP models \citep{lalchand2020approximate}.
Indeed, \citet{tran2019calibrating} showed that using Monte Carlo dropout to perform approximate Bayesian inference over the network parameters in DKL can improve calibration.\looseness=-1

\begin{figure}[t]
    \centering
    \begin{subfigure}[b]{0.23\textwidth}
    \centering
    \centerline{\includegraphics[width=\textwidth]{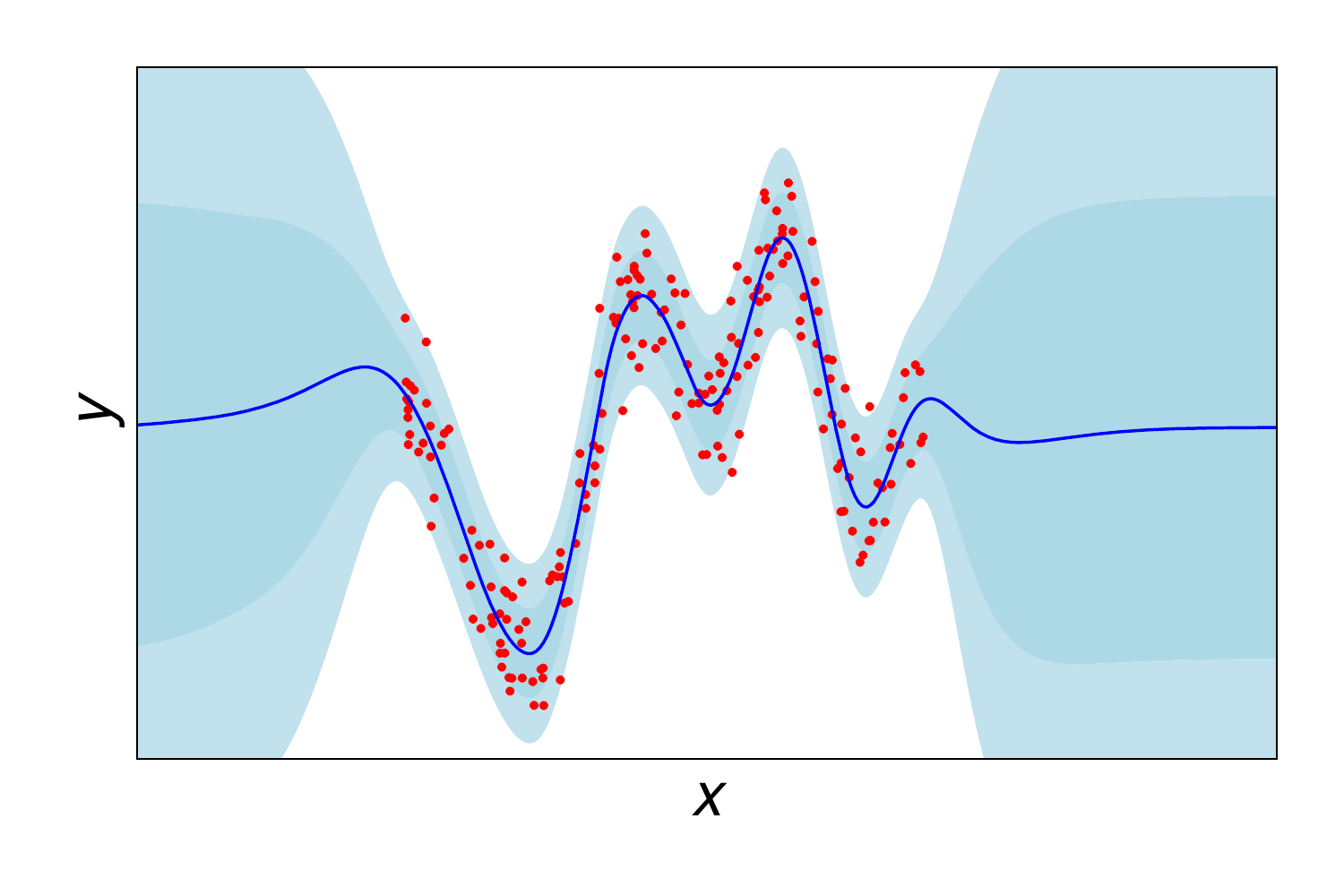}}
    \caption{Original dataset}
    \label{fig:toy_hmc_200}
    
    \end{subfigure}
    \hfill
    \begin{subfigure}[b]{0.23\textwidth}
    \centering
    \centerline{\includegraphics[width=\textwidth]{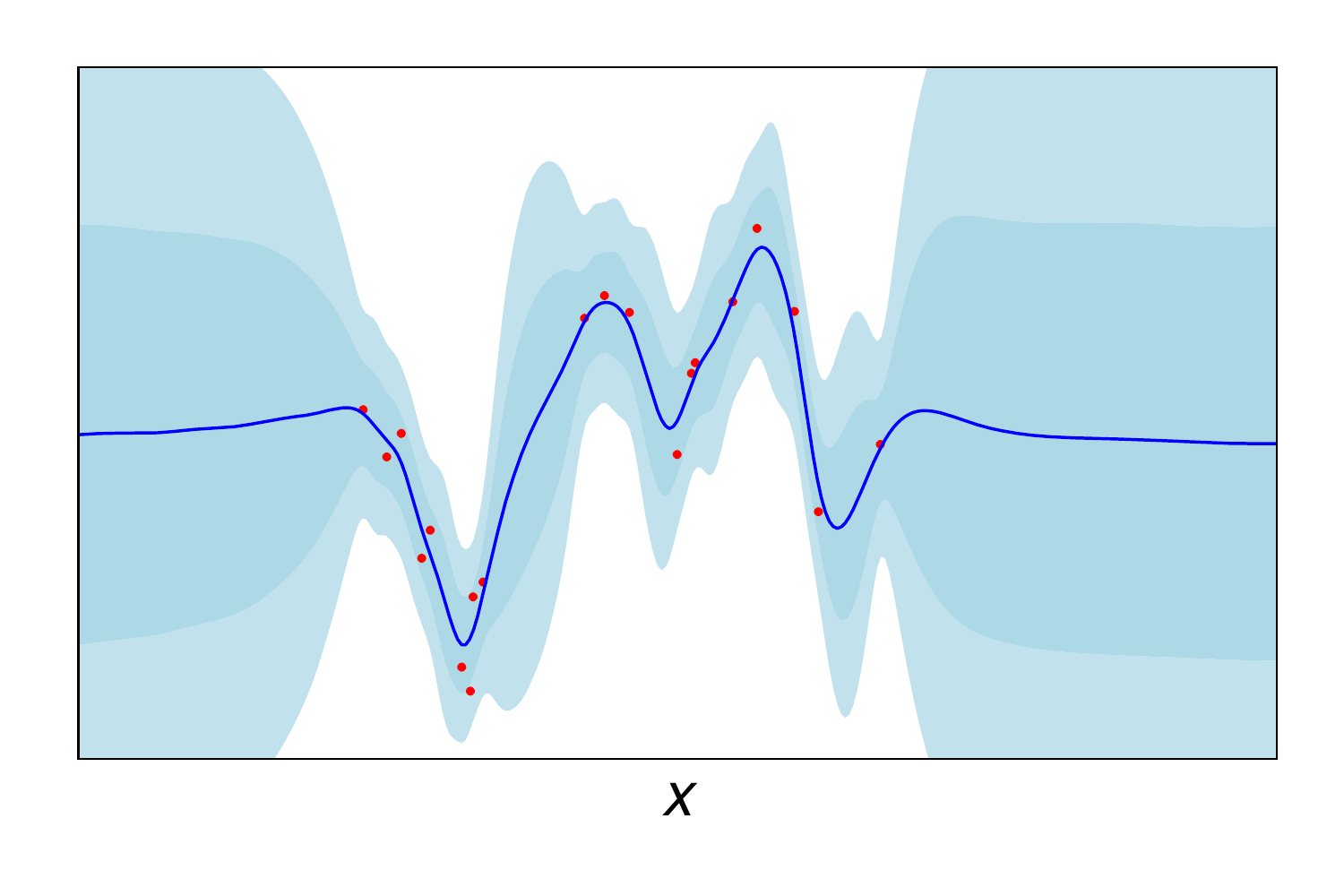}}
    \caption{Subsampled dataset}
    \label{fig:toy_hmc_20}
    
    \end{subfigure}
    \caption{Posteriors for fully Bayesian DKL using HMC.}
    \label{fig:hmc}
\end{figure}
\begin{table}
\small
    \centering
    \caption{Results for the image datasets with SGLD.}\label{tab:sgld}
    \begin{tabular}{rcc}
\toprule 
& NN & SVDKL \\ 
\midrule 
UTKFace - Test RMSE & $\mathbf{0.16 \pm 0.00}$ & $\mathbf{0.16 \pm 0.00}$ \\
Test LL & 0.39 $\pm$ 0.04 & $\mathbf{0.42 \pm 0.03}$ \\
\midrule
CIFAR-10 - Test Acc. & $\mathbf{0.79 \pm 0.00}$ & 0.78 $\pm$ 0.00 \\
Test LL & -1.89 $\pm$ 0.02 & -$\mathbf{1.11 \pm 0.02}$ \\
Inc. Test LL & -8.78 $\pm$ 0.11 & -$\mathbf{4.94 \pm 0.10}$ \\
ECE & 0.18 $\pm$ 0.00 & $\mathbf{0.13 \pm 0.00}$ \\
\bottomrule 
    \end{tabular}
\end{table}

We test this hypothesis using sampling methods.
We first consider the 1D toy problem, using HMC \citep{neal2011mcmc} to sample the neural network weights along with the other GP hyperparameters, using the marginal likelihood as the potential.
We plot the resulting posterior in Fig.~\ref{fig:toy_hmc_200}, and see that this completely resolves the problems observed earlier. In fact, the uncertainty in the outer regions is even greater than that given by the standard SE fit in Fig.~\ref{fig:toy_SE_fit}, while still concentrating where there is data.
We additionally consider a subsampled version of the dataset in Fig.~\ref{fig:toy_hmc_20}.
There is still no overfitting despite the small dataset size: for a comparison to the baseline SE kernel and DKL, see Fig.~1 in the Appendix.\looseness=-1

Unfortunately, HMC in its standard form does not scale to the larger datasets considered in Sec.~\ref{sec:image-datasets}, due to the necessity of calculating gradients over the entire dataset and the calculation of the acceptance probability.
Therefore, we consider stochastic gradient Langevin dynamics (SGLD; \citet{welling2011bayesian}), which allows us to use minibatches.
We note that SGLD has relatively little additional training cost compared to SGD, as it simply injects scaled Gaussian noise into the gradients; the main cost is in memory and at test time.
While we do not necessarily expect that this will be as accurate to the true posterior as HMC (see e.g.~\citet{johndrow2020no}), we hope that it will give insights into what the performance of a fully Bayesian approach would be.
We select a batch size of 100, and give test results for the NN and SVDKL for both UTKFace and CIFAR-10 without data augmentation in Table~\ref{tab:sgld}.
We see that for both datasets, the additional uncertainty significantly helps the NN models.
The improvement is significant for SVDKL for the UTKFace dataset, and while not so significant for CIFAR-10, we still observe slight improvements in log likelihoods and ECE, although at the expense of slightly lower test accuracy.
Moreover, the fully Bayesian SVDKL outperforms the Bayesian NN in nearly every metric, and significantly so for the uncertainty-related metrics.
In fact, for CIFAR-10, the original version of SVDKL (i.e. pSVDKL) outperforms the Bayesian NN for the uncertainty metrics, even for the larger batch size experiments.
Therefore, we arrive at our final remark:\looseness=-1
\begin{rem}\label{rem:fully-Bayesian}
A fully Bayesian approach to deep kernel learning can prevent overfitting and obtain the benefits of both neural networks and Gaussian processes.
\end{rem}\looseness=-1

\section{Conclusions}
In this work we have explored the performance of DKL in different regimes. 
We have shown that, while DKL models can achieve good performance, this is mostly because of implicit regularization due to stochastic minibatching rather than a better marginal likelihood.
This stochastic regularization appears to be stronger than that for plain neural networks.
Moreover, we have shown that when this stochastic regularization is limited, the performance can be worse than that of standard neural networks, with more overfitting and unstable training.
This is surprising, because DKL models are more Bayesian than deterministic neural networks, and so one might expect that they would be less prone to overfitting.
However, we have shown that for highly parameterized models, the marginal likelihood tries to correlate all the datapoints rather than those that should be correlated: therefore, a higher marginal likelihood does not improve performance.
This means that when the number of hyperparameters is large, the marginal likelihood cannot be relied upon for model selection as it often is, just as the standard maximum likelihood training loss cannot be used for model selection.
Finally, we showed that a fully Bayesian approach to the neural network hyperparameters can overcome this limitation and improve the performance over the less Bayesian approach, fully showing the advantages of DKL models.\looseness=-1

\begin{acknowledgements}
We thank the anonymous reviewers, John Bradshaw, David R. Burt, Andrew Y.K. Foong, Markus Lange-Hegermann, James Requeima, Pola Schw\"{o}bel, and Andrew Gordon Wilson for helpful discussions and comments. SWO acknowledges the Gates Cambridge Trust for funding his PhD.\looseness=-1
\end{acknowledgements}

\bibliography{uai2021-template}

\onecolumn

\appendix
\section{Proof of Proposition 1}\label{app:prop-proof}
We restate the proposition:
\lmlprop*

\begin{proof}
We reparameterize $\sigma_n^2 = \hat{\sigma}_n^2\sigma_f^2$. Then, writing $K + \sigma_n^2 I_N = \sigma_f^2(\hat{K} + \hat{\sigma}_n^2 I_N)$, the result follows by differentiating the log marginal likelihood with respect to $\sigma_f^2$:
\begin{align*}
    \frac{d}{d\sigma_f^2} \log p(\y) &= \frac{d}{d\sigma_f^2} \left(-\frac{N}{2}\log \sigma_f^2 -\frac{1}{2}\log |\hat{K}+\hat{\sigma}_n^2 I_N| - \frac{1}{2\sigma_f^2}\y^T(\hat{K} + \hat{\sigma}_n^2 I_N)^{-1}\y\right)\\
    &= -\frac{N}{2\sigma_f^2} + \frac{1}{2\sigma_f^4}\y^T(\hat{K}+\hat{\sigma}_n^2 I_N)^{-1}\y.
\end{align*}
Setting the derivative equal to zero gives:
\begin{equation*}
    \sigma_f^2 = \frac{1}{N}\y^T(\hat{K} + \hat{\sigma}_n^2 I_N)^{-1}\y.
\end{equation*}
Substituting this into the data fit term gives the desired result.
\end{proof}

We note that this result was essentially proven in \citet{moore2016fast}, although they did not consider the last step of substituting the result into the data fit term.
Instead, they used the result as a means of analytically solving for the optimal signal variance to reduce the number of parameters and hence speed up optimization.

\section{Experimental Details}\label{app:exp-deets}
All experiments on real datasets (MNIST, UCI, CIFAR-10, UTKFace) were written in TensorFlow 2 \citep{tensorflow2015-whitepaper}, using GPflow \citep{GPflow2017} to implement the DKL models.
We use jug \citep{coelho2017jug} to easily run the experiments.
The experiments were run on single GPUs using both NVIDIA Tesla P100-PCIE-16GB GPUs and NVIDIA GeForce RTX 2080 Ti GPUs.

\subsection{Datasets}\label{app:exp-datasets}
We describe the datasets used as well as the splits and preprocessing.
\paragraph{Toy dataset} The toy dataset is that as introduced in \citet{snelson2006sparse}. The dataset comprises 200 input-output pairs and can be found at \url{http://www.gatsby.ucl.ac.uk/~snelson/}. We normalize both inputs and outputs for training and plot the unnormalized values and predictions.

\paragraph{MNIST} We take the first 5,000 datapoints from the standard MNIST dataset \citep{lecun2010mnist}, and use the standard 10,000 point test set for evaluation.
We preprocess the images by dividing the pixel values by 255.

\paragraph{UCI} We use a slightly modified version of Bayesian Benchmarks (\url{https://github.com/hughsalimbeni/bayesian_benchmarks}) to obtain the UCI datasets we use. 
The modification is to rectify minor data leakage in the normalization code: they normalize using the statistics from the entire dataset before dividing into train/test splits, instead of normalizing using only the train split statistics. 
We perform cross-validation using 20 90\%/10\% train/test splits, and report means and standard errors for each metric.
Note that we report metrics on the normalized datasets to lead to more interpretable results: namely, an RMSE of 1 corresponds to predicting 0 for each test point.

\paragraph{CIFAR-10} We use the standard CIFAR-10 dataset \citep{krizhevsky2009learning}, with the standard train-validation split of 50,000 and 10,000 images, respectively, using the validation split as the test set, as is common practice. 
We preprocess the images by simply dividing the pixel values by 255, so that each value lies between 0 and 1. 

\paragraph{UTKFace} The UTKFace dataset \citep{zhifei2017cvpr} is a large face dataset consisting of 23,708 images of faces, annotated with age, gender, and ethnicity. 
The faces have an age range from 0 to 116.
We use the aligned and cropped version to limit the amount of preprocessing necessary, available at \url{https://susanqq.github.io/UTKFace/}.
These cropped images have sizes $200\times 200\times 3$.
We choose 20,000 images to be in the train dataset, with the remaining being used for testing.
We again perform preprocessing by dividing the pixel values by 255, and we additionally normalize the age values.
The metrics we report all use the normalized values, as with the UCI datasets.

\subsection{Models}\label{app:exp-models}
We describe the models used for the experiments. 
To ensure that the comparisons between neural networks and DKL models are as fair as possible, we ensure that each model used in direct comparison has the same number of layers: for the DKL models, we remove the last fully-connected layer of the neural network and replace it with the ARD SE GP. 
All neural networks use ReLU activations.
For all SVDKL models, the inducing points live in the neural network feature space at the input to the GP.

\paragraph{Toy dataset} We use an architecture of $[100, 50, 2]$ for the hidden-layer widths for the neural network.
For DKL, we use the pre-activation features of the final hidden layer for the input to the GP.

\paragraph{MNIST} We use a large fully-connected ReLU architecture of $[1000, 500, 500, 100, 100, 50, 50, 10]$.
For the DKL model, we use the same feature extractor with an ARD SE kernel, and 5000 inducing points to minimize the bias from the variational approximation.
The inducing points are initialized using the RobustGP method from \citet{burt2020convergence}.
We use the softmax likelihood for all models.

\paragraph{UCI} For \textsc{Boston, Energy}, we use a ReLU architecture of $[50, 50]$, and a ReLU architecture $[1000, 500, 50]$ for \textsc{Kin40K, Power}, and \textsc{Protein}.
We note that these architectures are smaller than the ones proposed by \citet{wilson2016deep}
For the SVGP baseline, we use an ARD SE kernel, with 100 inducing points for the small datasets (\textsc{Boston}, \textsc{Energy}) and 1000 inducing points for the larger ones.
For the DKL models, we use the post-activation features from the final hidden layer as inputs to the GPs, which use ARD SE kernels.
For SVDKL, we initialize the inducing points using the k-means algorithm on a subset of the training set.
We use 100 inducing points for the smaller datasets (\textsc{Boston, Energy}) and 1000 on the larger ones.
The method for initializing the inducing points, and the number of inducing points, is the same for the SVGP baseline model, which uses a standard ARD SE kernel.

\paragraph{CIFAR-10} We use a modified ResNet18 \citep{he2016deep} architecture as the baseline neural network architecture; the main modification is that we have added another fully-connected layer at the output to ensure that the neural network and DKL models are comparable in depth.
Therefore, instead of the standard single fully-connected layer after a global average pooling layer, we have two fully-connected layers.
While we could take the output of the global average pooling layer, this is typically very high-dimensional and thus potentially unsuitable as an input to a GP.
For most experiments, we fix the width of the last hidden layer (the final feature width) to 10, although we do consider changing that in App.~\ref{app:add_exp}.
We additionally add batchnorm layers \citep{ioffe2015batch} before the ReLU activations in the residual blocks.
For SVDKL, we use 1000 inducing points initialized with k-means on a subset of the training set.
As with UCI, the features at the input to the GP are post-activation features.
For all classification models, we use softmax activation to obtain probabilities for the cross-entropy loss, and for the SVDKL models we use 10 samples from the latent function posterior to compute the log likelihood term of the ELBO.

\paragraph{UTKFace} As with CIFAR-10, we again use a modified ResNet18 \citep{he2016deep} architecture with an additional fully-connected layer at the output.
For SVDKL, we again use 1000 inducing points.

\subsection{Implementation Details}\label{app:imp-details}
All models are optimized using ADAM \citep{kingma2014adam}.
Throughout, we try to ensure that we train each model for comparable numbers of gradient steps and learning rates.

\paragraph{Toy dataset} We train both the NN and DKL models in for 10,000 gradient steps using learning rates of 0.001. No weight decay was used. For the HMC experiments, we use a step size of 0.005, 20 leapfrog steps, and a prior variance of 1 on the network weights. We burn in for 10,000 samples, then use 1,000 iterations to sample, thinned by a factor of 10.

\paragraph{MNIST} We train all models using full batch training, i.e. a batch size of 5000. 
For the pretraining of the feature extractor, we use 96,000 gradient steps (corresponding to 160 epochs of training full MNIST with batch size 100), with an initial learning rate of 1e-3 and no weight decay.
We incorporate learning rate steps at halfway and three quarters through the training, stepping down by a factor of 10 each time.
For the neural network after pretraining and for fDKL, we use the same procedure.
For the DKL model, we first train only the variational and ARD SE parameters for 9,600 gradient steps (corresponding to 16 epochs of training full MNIST with batch size 100), with no learning rate schedule.
We then train everything jointly for 86,400 gradient steps (corresponding to 144 epochs), starting again at a learning rate of 1e-3 and decreasing the learning rate by a factor of 10 at the halfway and three quarters mark.

\paragraph{UCI} For each model, we train with an initial learning rate of 0.001. For \textsc{Boston} and \textsc{Energy}, we use a batch size of 32 and train the minibatched algorithms for a total of 400 epochs, and use a learning rate scheduler that decreases the learning rate by a factor of 10 after 200 and 300 epochs. For \textsc{Kin40K}, \textsc{Power}, and \textsc{Protein} we use a batch size of 100, training for 160 epochs with the same learning rate schedule that triggers at 80 and 120 epochs. For the full-batch methods, we ensure that they are trained for the lesser of the same number of gradient steps or 8000 gradient steps (due to limited computational budget), with the learning rate schedule set to trigger at the corresponding number of gradient steps as the batched methods. This ensures a fair comparison when claiming that the full-batch methods overfit in comparison to the stochastic versions. For the deep models, we use a weight decay of 1e-4 on the neural network weights. We do not use any pretraining for the DKL models, as we did not find it necessary for these datasets. We initialize the log noise variance to -4 for the DKL models. We train the neural network models using mean squared error loss, and use the maximum likelihood noise estimate after training to compute train and test log likelihoods.

\paragraph{CIFAR-10} We describe the details for batch size 100; for batch size 500, we ensure that we use the same number of gradient steps. We do note use weight decay as we found that it hurt test accuracy. For NN and SVDKL, we train for 160 epochs total: we decrease the learning rate from the initial 1e-3 by a factor of 10 at 80 and then 120 epochs. For pNN, we train for an additional 160 epochs in the same way (restarting the learning rate at 1e-3). For pSVDKL, we start by training with the neural network parameters fixed for 80 epochs, with learning rate decreases at 40 and 60 epochs. We then reset the learning rate to 1e-3, and train for an additional 80 epochs with learning rate decreases at 40 and 60 epochs. For the experiments with data augmentation, we use random horizontal flipping and randomly crop $32\times32\times3$ images from the original images padded up to $40\times40\times3$.

We use the same losses as the potentials for SGLD. We use the trained NNs to initialize the weights to reasonable values, and set the batch size to 100. We initialize the learning rate to 1e-3 (which we then scale down by the dataset size to account for the scale of the potential), and decay the learning rate at each epoch by a factor of $1/(1 + 0.4\times \text{epoch})$ to satisfy Robbins-Monro. For the NN, we burn in for 100 epochs, and then sample every other epoch for 100 epochs, leading to 50 samples. For SVDKL, we follow the approach in \citet{hensman2015mcmc}, and learn the variational parameters (i.e. 1000 inducing points) and GP hyperparameters with the fixed, pretrained NN weights, using the same hyperparameters as for pSVDKL. We then follow the SGLD approach we took for the NN, with 100 epochs of burn in and 100 epochs of sampling, starting with a learning rate of 1e-3.

\paragraph{UTKFace} We follow the same approach as for CIFAR-10. We list the minor differences. We use a small weight decay of 1e-4. For the SVDKL models, we initialize the log noise variance to -4. We use mean squared error loss for the NNs; however, for SGLD we use a Gaussian likelihood with log noise variance initialized to -4. For the SGLD experiments, we initialize the learning rate to 1e-5. For data augmentation, we again use random horizontal flipping as well as randomly cropping $200\times200\times3$ images from the original images padded up to $240\times240\times3$.

\section{Additional Experimental Results}\label{app:add_exp}
Here we briefly present some additional experimental results.

\subsection{Toy}\label{app:add_toy}
\begin{figure*}[t]
     \centering
     \begin{subfigure}[b]{0.32\textwidth}
         \centering
         \includegraphics[width=\textwidth]{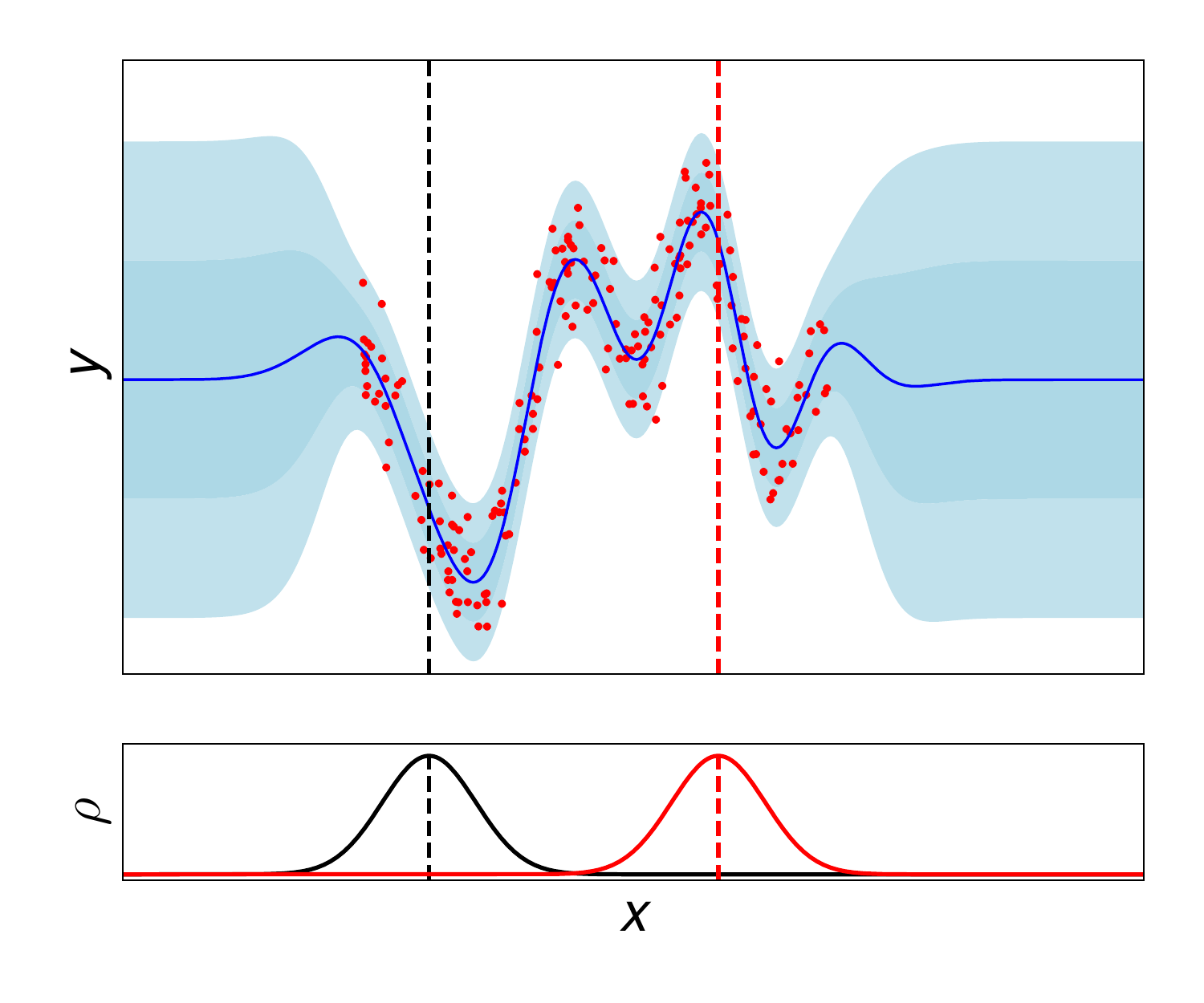}
         \caption{SE kernel. $\mathrm{LML}=-89.3$}
         \label{fig:app-RBF_200}
     \end{subfigure}
     \hfill
     \begin{subfigure}[b]{0.32\textwidth}
         \centering
         \includegraphics[width=\textwidth]{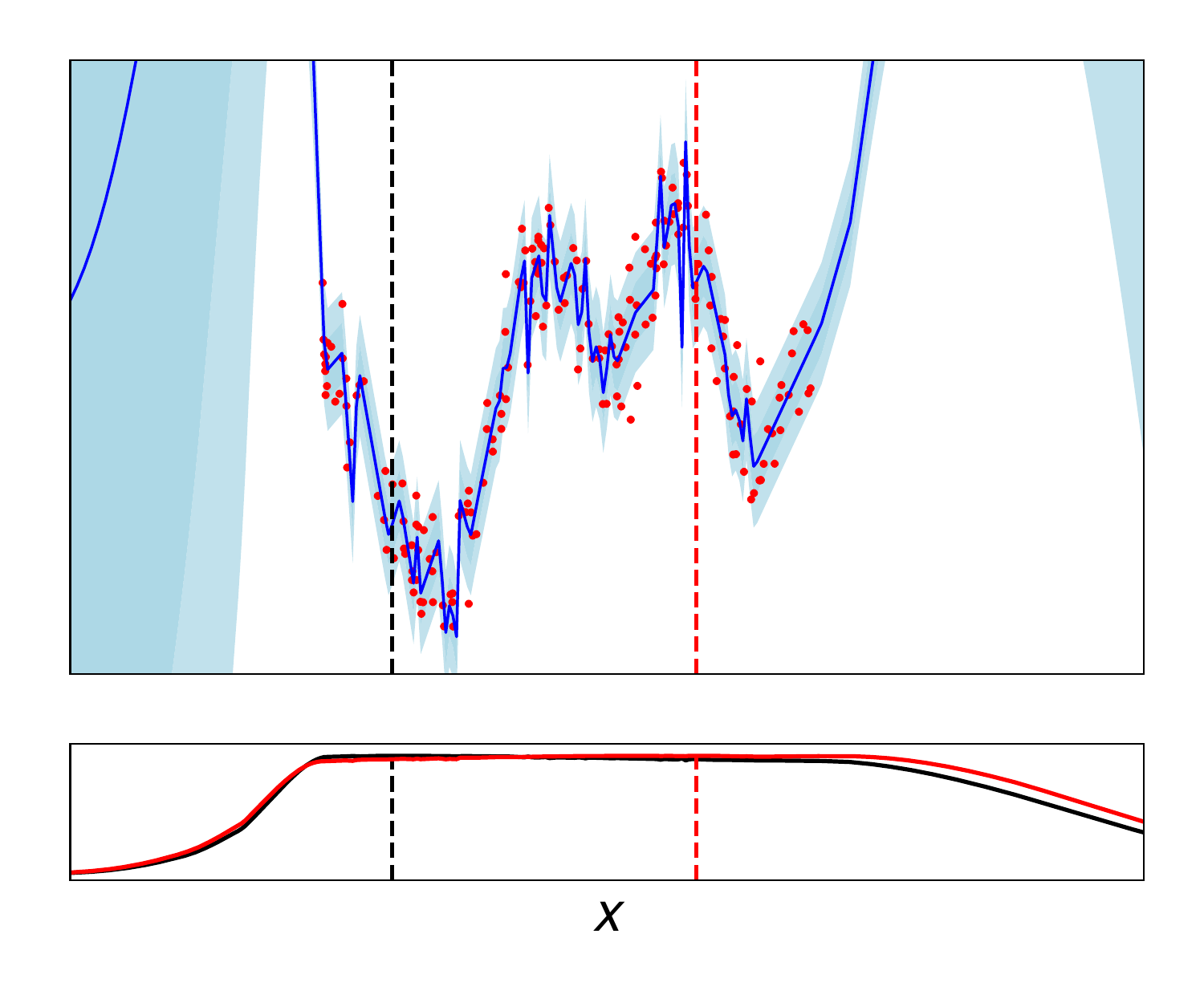}
         \caption{Example DKL. $\mathrm{LML} = -23.4$}
         \label{fig:app-DKL_200_0}
     \end{subfigure}
     \hfill
     \begin{subfigure}[b]{0.32\textwidth}
         \centering
         \includegraphics[width=\textwidth]{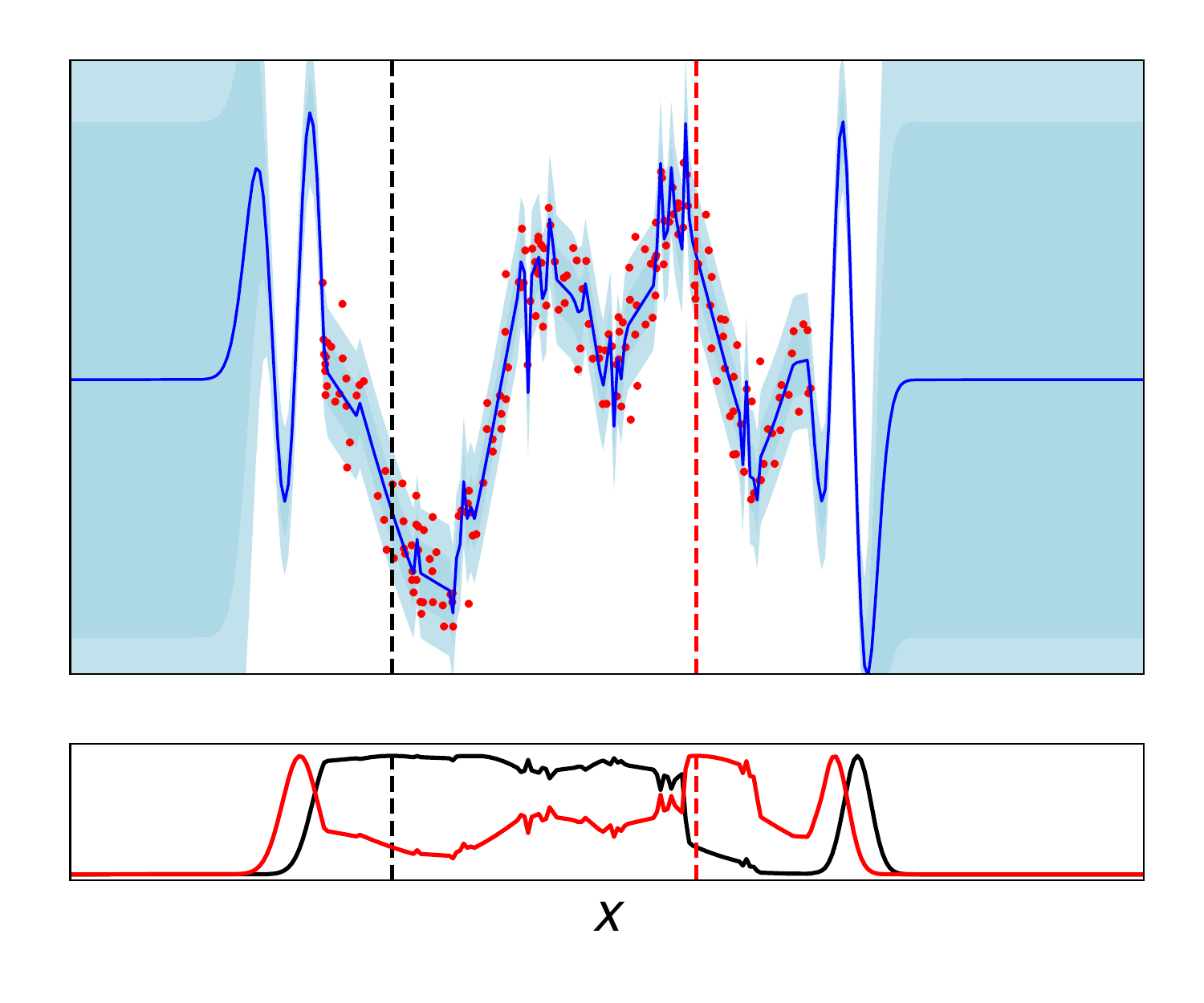}
         \caption{Example DKL. $\mathrm{LML} = -45.7$}
         \label{fig:app-DKL_200_1}
     \end{subfigure}

     \begin{subfigure}[b]{0.32\textwidth}
         \centering
         \includegraphics[width=\textwidth]{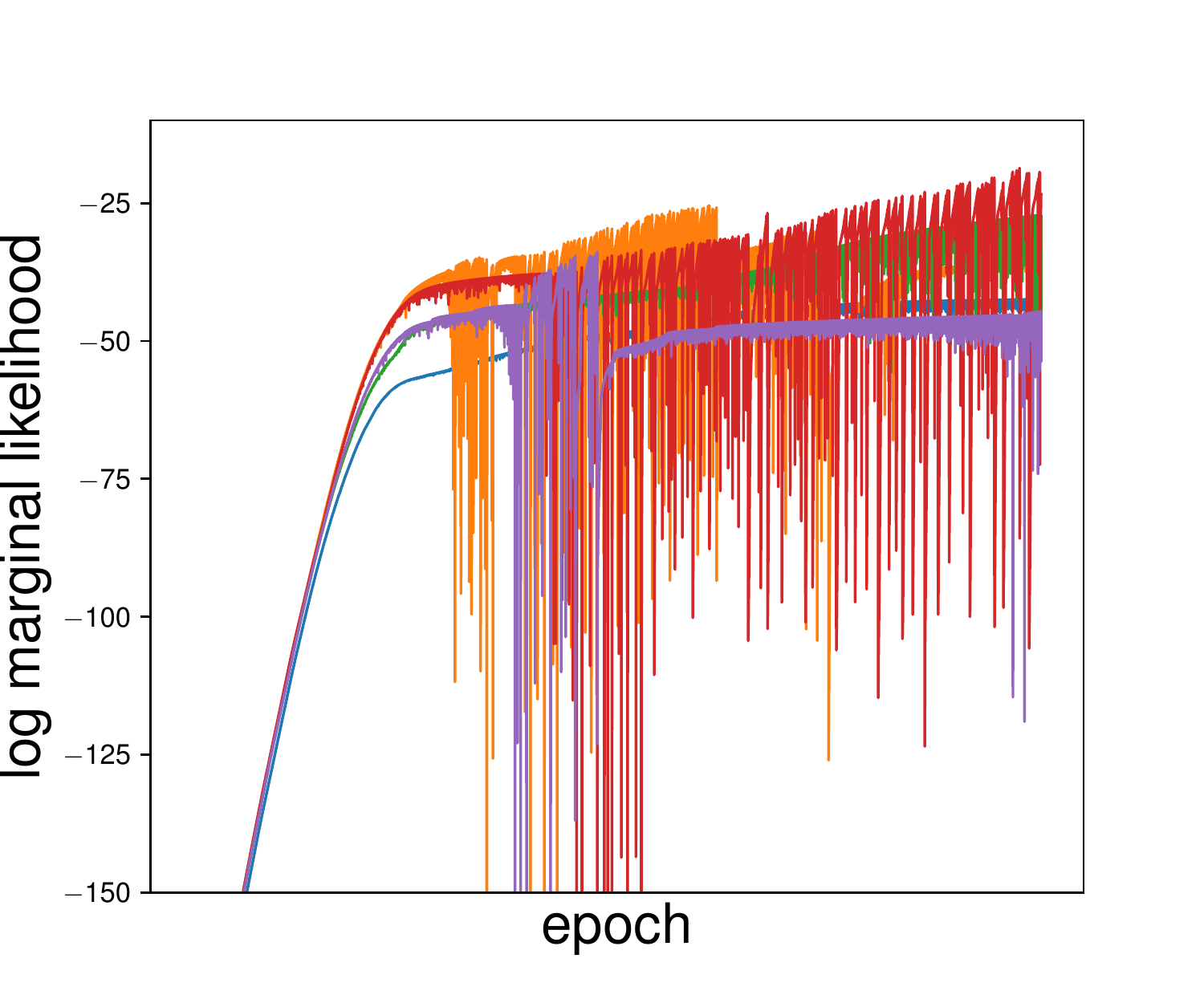}
         \caption{Training curves for 5 different initializations for DKL on the full toy dataset}
         \label{fig:app-training_curves}
     \end{subfigure}
     \hfill
     \begin{subfigure}[b]{0.32\textwidth}
         \centering
         \includegraphics[width=\textwidth]{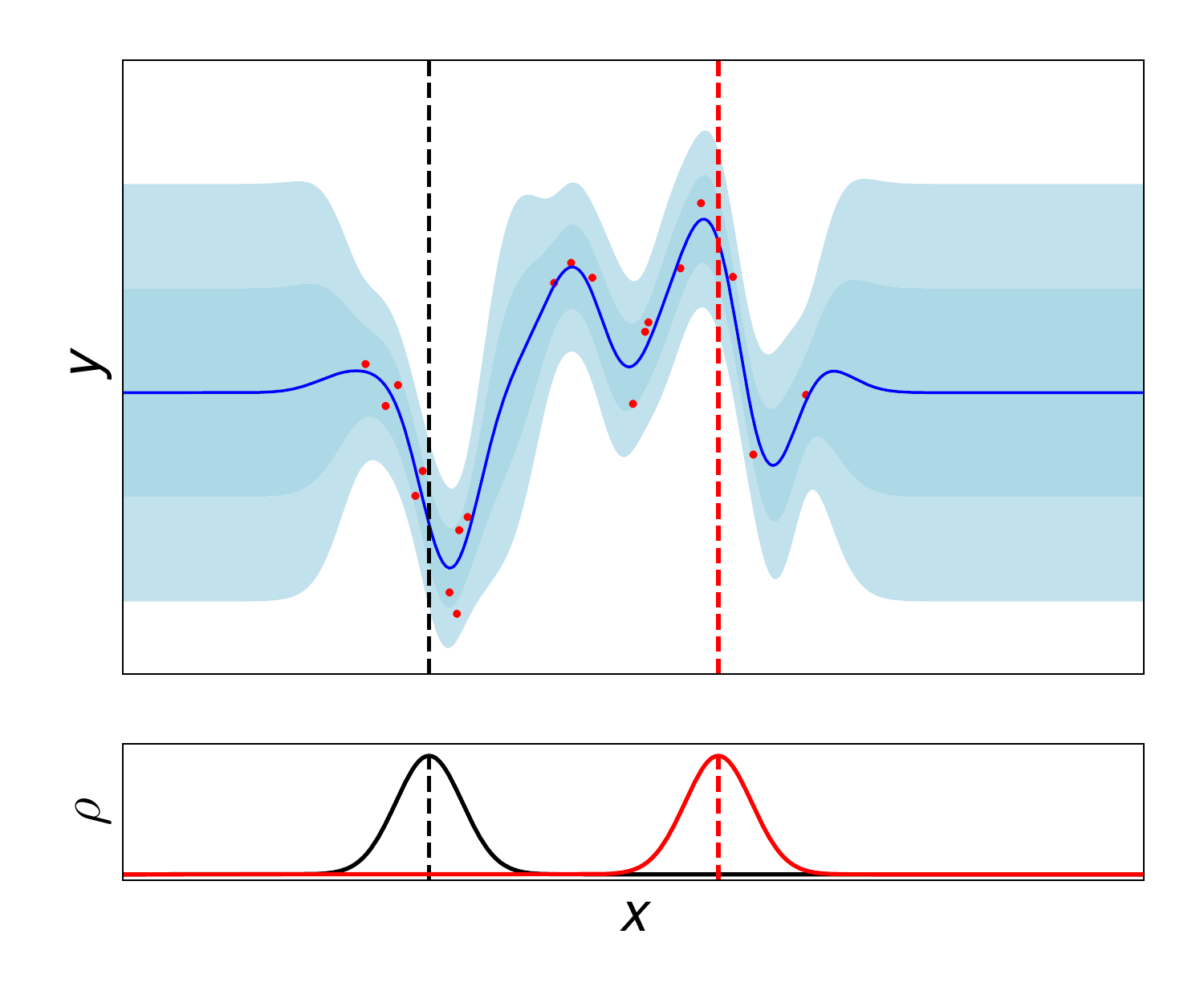}
         \caption{SE kernel. $\mathrm{LML} = -17.5$}
         \label{fig:app-RBF_20}
     \end{subfigure}
     \hfill
     \begin{subfigure}[b]{0.32\textwidth}
         \centering
         \includegraphics[width=\textwidth]{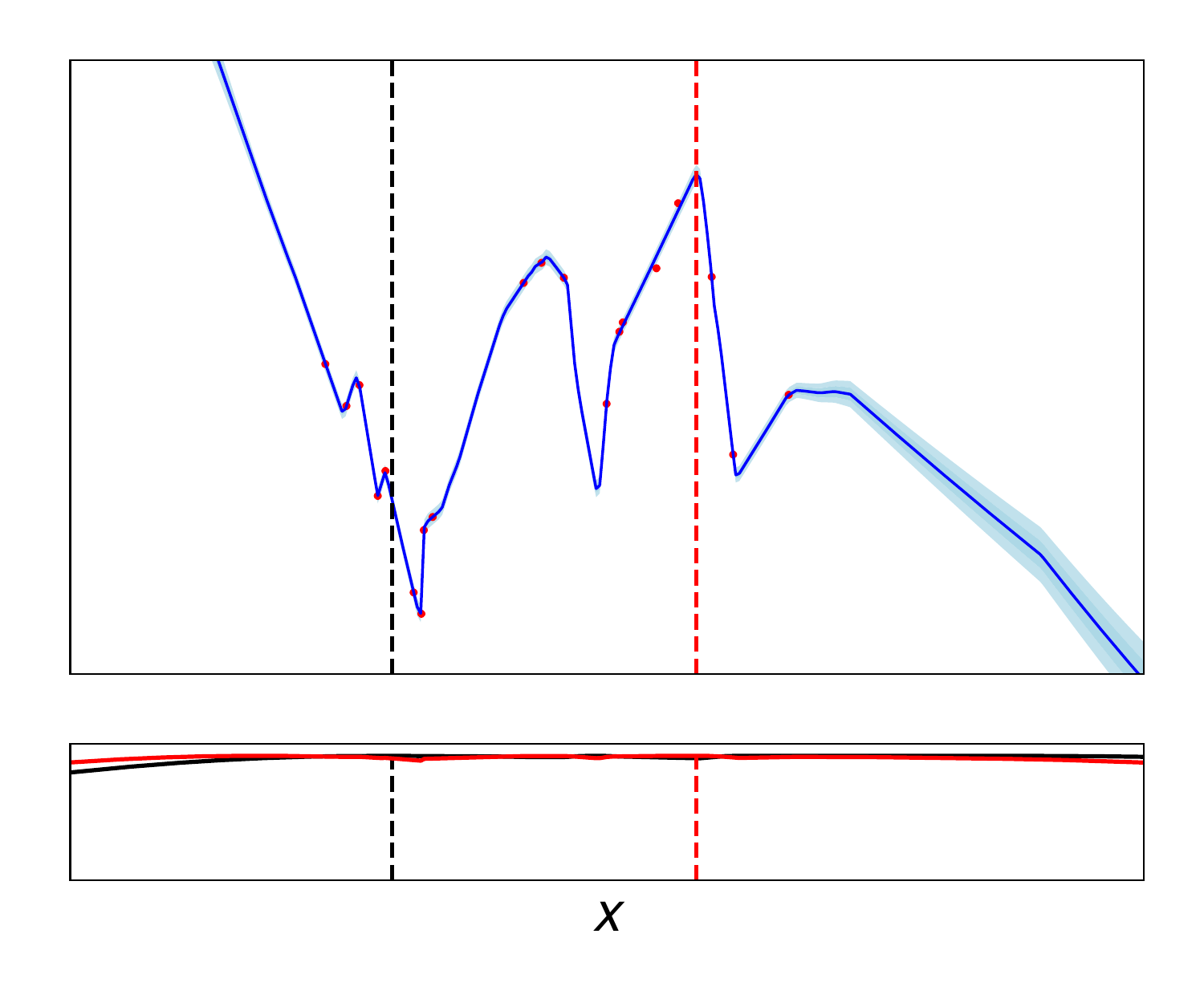}
         \caption{Example DKL. $\mathrm{LML}=28.3$}
         \label{fig:app-DKL_20}
     \end{subfigure}
        \caption{Plots of fits and training curves using standard SE kernel and DKL. Below each fit we plot two correlation functions $\rho_{x'}(x) = k(x, x')/\sigma_f^2$ induced by each kernel, where the location of $x'$ is given by the dashed vertical lines. (a)-(d) show fits and training curves for the full toy dataset introduced in \citet{snelson2006sparse}, whereas (e)-(f) show fits on the subsampled version from \citet{titsias2009variational}.}
        \label{fig:app-toy}
\end{figure*}
We show additional plots of fits and training curves for the toy problem from \citet{snelson2006sparse} in Fig.~\ref{fig:app-toy}. Below the fits, we again show the kernel correlation at two different points $x'$, marked by the vertical dashed lines. In Figure~\ref{fig:app-RBF_200}, we show the fit using the standard squared exponential kernel, followed by two fits using DKL in Figures~\ref{fig:app-DKL_200_0} and~\ref{fig:app-DKL_200_1}. In Fig.~\ref{fig:app-training_curves} we show training curves for 5 different initializations; note that unlike in the main text we use a learning rate of 1e-4 here and so require more training iterations to converge. Finally, we consider plots of fits on the subsampled version of the dataset as in \citet{titsias2009variational}; we show fits using the SE kernel and DKL in Figures~\ref{fig:app-RBF_20} and~\ref{fig:app-DKL_20}, respectively. For each fit, we also show the log marginal likelihood in the caption.

We make a few observations. First, note that different initializations can lead to very different fits and LMLs. Moreover, as predicted by our theory, the highest log marginal likelihoods are obtained when the prior attempts to correlate all the points in the input domain: Fig.~\ref{fig:app-DKL_200_0} obtains a higher LML than Fig.~\ref{fig:app-DKL_200_1}. However, instability in training often leads to worse LMLs than could be obtained (Fig.~\ref{fig:app-training_curves}). Finally, we note that the overfitting is substantially worse on the subsampled version of the dataset: we also see that the prior is more correlated than previously (Fig.~\ref{fig:app-DKL_20}).

\subsection{Changing the Feature Dimension}\label{app:fd}
We perform experiments changing the feature dimension $Q$ for the UTKFace and CIFAR-10 datasets. We present the results in Tables~\ref{tab:fd_age} and~\ref{tab:fd_cifar}, where each model name is followed by the feature space dimension. For UTKFace, it is clear that 2 neurons is not sufficient to fit the data. Beyond 2, we see only minor changes in performance. For CIFAR-10, we find that we need at least 10 neurons to fit well, but beyond 10 there are again only minor differences. We choose 10 neurons for both experiments out of convenience.

\begin{table}
    \centering
    \caption{Results for UTKFace.}\label{tab:fd_age}
    \begin{tabular}{rccccc}
\toprule 
& ELBO &  Train RMSE & Test RMSE & Train LL & Test LL \\ 
\midrule 
NN-2 & - & 0.37 $\pm$ 0.26 & 0.61 $\pm$ 0.16 & 0.58 $\pm$ 0.82 & -23.14 $\pm$ 10.94 \\ 
SVDKL-2& -0.22 $\pm$ 1.00 & 0.39 $\pm$ 0.25 & 0.61 $\pm$ 0.16 & 0.14 $\pm$ 0.68 & -3.60 $\pm$ 1.90 \\ 
pNN-2 & - & 0.36 $\pm$ 0.26 & 0.61 $\pm$ 0.16 & 0.74 $\pm$0.88 & -35.66 $\pm$ 14.00 \\ 
fSVDKL-2& -0.32 $\pm$ 4.00 & 0.45 $\pm$ 0.22 & 0.45 $\pm$0.23 & -0.26 $\pm$ 0.50 & -0.23 $\pm$0.52 \\ 
pSVDKL-2& 0.36 $\pm$3.00 & 0.36 $\pm$0.26 & 0.53 $\pm$ 0.20 & 0.47 $\pm$ 0.77 & -3.24 $\pm$ 0.76 \\ 
\midrule 
NN-5 & - & 0.04 $\pm$0.00 & 0.41 $\pm$ 0.01 & 1.72 $\pm$ 0.05 & -42.86 $\pm$ 2.79 \\ 
SVDKL-5& -0.17 $\pm$ 0.59 & 0.11 $\pm$ 0.01 & 0.47 $\pm$ 0.01 & 0.40 $\pm$ 0.11 & -1.41 $\pm$ 0.16 \\ 
pNN-5 & - & 0.04 $\pm$0.00 & 0.41 $\pm$ 0.00 & 1.79 $\pm$0.01 & -48.92 $\pm$ 0.62 \\ 
fSVDKL-5& 0.31 $\pm$ 0.48 & 0.17 $\pm$0.00 & 0.17 $\pm$0.00 & 0.38 $\pm$ 0.02 & 0.37 $\pm$0.03 \\ 
pSVDKL-5& 0.99 $\pm$0.73 & 0.04 $\pm$0.00 & 0.32 $\pm$ 0.01 & 1.18 $\pm$ 0.06 & -2.61 $\pm$ 0.26 \\ 
\midrule 
NN-10 & - & 0.04 $\pm$0.00 & 0.40 $\pm$ 0.00 & 1.81 $\pm$ 0.01 & -48.73 $\pm$ 1.64 \\ 
SVDKL-10& 0.92 $\pm$ 0.15 & 0.04 $\pm$0.00 & 0.40 $\pm$ 0.01 & 1.30 $\pm$ 0.01 & -6.88 $\pm$ 0.38 \\ 
pNN-10 & - & 0.04 $\pm$0.00 & 0.41 $\pm$ 0.00 & 1.83 $\pm$0.01 & -53.72 $\pm$ 1.71 \\ 
fSVDKL-10& 1.05 $\pm$0.02 & 0.08 $\pm$ 0.03 & 0.31 $\pm$0.07 & 1.16 $\pm$ 0.31 & -7.55 $\pm$ 3.42 \\ 
pSVDKL-10& 1.03 $\pm$ 0.07 & 0.04 $\pm$0.00 & 0.38 $\pm$ 0.02 & 1.20 $\pm$ 0.08 & -4.74 $\pm$1.35 \\ 
\midrule 
NN-20 & - & 0.04 $\pm$0.00 & 0.40 $\pm$ 0.00 & 1.78 $\pm$ 0.01 & -46.77 $\pm$ 1.72 \\ 
SVDKL-20& 0.22 $\pm$ 0.01 & 0.08 $\pm$ 0.02 & 0.43 $\pm$ 0.01 & 0.78 $\pm$ 0.22 & -3.42 $\pm$1.52 \\ 
pNN-20 & - & 0.04 $\pm$0.00 & 0.41 $\pm$ 0.00 & 1.80 $\pm$0.01 & -50.51 $\pm$ 0.47 \\ 
fSVDKL-20& 0.71 $\pm$ 0.30 & 0.12 $\pm$ 0.03 & 0.24 $\pm$0.06 & 0.83 $\pm$ 0.34 & -5.06 $\pm$ 4.46 \\ 
pSVDKL-20& 1.15 $\pm$0.10 & 0.04 $\pm$0.00 & 0.34 $\pm$ 0.03 & 1.33 $\pm$ 0.04 & -4.16 $\pm$ 0.42 \\ 
\midrule 
NN-50 & - & 0.04 $\pm$0.00 & 0.40 $\pm$ 0.00 & 1.79 $\pm$ 0.01 & -47.72 $\pm$ 0.68 \\ 
SVDKL-50& 0.92 $\pm$ 0.27 & 0.04 $\pm$0.00 & 0.40 $\pm$ 0.00 & 1.33 $\pm$ 0.03 & -7.35 $\pm$ 0.45 \\ 
pNN-50 & - & 0.04 $\pm$0.00 & 0.41 $\pm$ 0.01 & 1.81 $\pm$0.00 & -51.14 $\pm$ 1.33 \\ 
fSVDKL-50& 1.14 $\pm$ 0.31 & 0.08 $\pm$ 0.03 & 0.32 $\pm$0.06 & 1.29 $\pm$ 0.35 & -11.60 $\pm$ 4.91 \\ 
pSVDKL-50& 1.21 $\pm$0.03 & 0.04 $\pm$0.00 & 0.37 $\pm$ 0.02 & 1.37 $\pm$ 0.02 & -5.71 $\pm$0.55 \\ 
\bottomrule
    \end{tabular}
\end{table}

\begin{table}
    \centering
    \small
    \caption{Results for CIFAR-10.}\label{tab:fd_cifar}
    \begin{tabular}{rccccccc}
\toprule 
& ELBO &  Train Acc. & Test Acc. & Train LL & Test LL & Inc. Test LL & ECE \\ 
\midrule 
NN-2 & - & 0.69 $\pm$ 0.24 & 0.51 $\pm$ 0.17 & -0.81 $\pm$ 0.61 & -6.42 $\pm$ 2.23 & -6.40 $\pm$ 1.78 & 0.14 $\pm$ 0.06 \\ 
SVDKL-2& -1.38 $\pm$ 0.38 & 0.52 $\pm$ 0.17 & 0.46 $\pm$ 0.15 & -1.34 $\pm$ 0.39 & -1.57 $\pm$0.30 & -2.61 $\pm$0.13 & 0.04 $\pm$0.02 \\ 
pNN-2 & - & 0.70 $\pm$0.24 & 0.53 $\pm$0.17 & -0.77 $\pm$0.63 & -5.25 $\pm$ 1.36 & -7.28 $\pm$ 2.03 & 0.15 $\pm$ 0.06 \\ 
fSVDKL-2& -0.85 $\pm$ 0.59 & 0.69 $\pm$ 0.24 & 0.51 $\pm$ 0.17 & -0.81 $\pm$ 0.61 & -1.83 $\pm$ 0.20 & -4.35 $\pm$ 0.84 & 0.14 $\pm$ 0.06 \\ 
pSVDKL-2& -0.78 $\pm$ 0.62 & 0.70 $\pm$0.24 & 0.53 $\pm$0.17 & -0.78 $\pm$ 0.62 & -1.76 $\pm$ 0.23 & -4.51 $\pm$ 0.90 & 0.13 $\pm$ 0.05 \\ 
\midrule 
NN-5 & - & 0.70 $\pm$0.25 & 0.55 $\pm$0.18 & -0.77 $\pm$0.63 & -3.03 $\pm$ 0.58 & -7.00 $\pm$ 1.92 & 0.13 $\pm$ 0.05 \\ 
SVDKL-5& -1.66 $\pm$ 0.26 & 0.40 $\pm$ 0.12 & 0.39 $\pm$ 0.12 & -1.63 $\pm$ 0.28 & -1.68 $\pm$ 0.25 & -2.30 $\pm$ 0.00 & 0.02 $\pm$ 0.01 \\ 
pNN-5 & - & 0.70 $\pm$ 0.25 & 0.55 $\pm$ 0.18 & -0.77 $\pm$ 0.63 & -3.22 $\pm$ 0.68 & -7.34 $\pm$ 2.06 & 0.13 $\pm$ 0.05 \\ 
fSVDKL-5& -0.79 $\pm$ 0.62 & 0.70 $\pm$ 0.24 & 0.55 $\pm$ 0.18 & -0.77 $\pm$ 0.63 & -1.61 $\pm$ 0.29 & -4.25 $\pm$ 0.80 & 0.09 $\pm$ 0.04 \\ 
pSVDKL-5& -0.77 $\pm$ 0.62 & 0.70 $\pm$ 0.24 & 0.55 $\pm$ 0.19 & -0.77 $\pm$ 0.63 & -1.64 $\pm$ 0.27 & -4.66 $\pm$ 0.96 & 0.11 $\pm$ 0.04 \\ 
\midrule 
NN-10 & - & 1.00 $\pm$ 0.00 & 0.79 $\pm$ 0.00 & -0.00 $\pm$ 0.00 & -2.05 $\pm$ 0.03 & -8.87 $\pm$ 0.10 & 0.18 $\pm$ 0.00 \\ 
SVDKL-10& -0.76 $\pm$ 0.28 & 0.76 $\pm$ 0.09 & 0.63 $\pm$ 0.03 & -0.71 $\pm$ 0.28 & -1.37 $\pm$ 0.10 & -3.38 $\pm$ 0.77 & 0.10 $\pm$ 0.05 \\ 
pNN-10 & - & 1.00 $\pm$ 0.00 & 0.79 $\pm$ 0.00 & -0.00 $\pm$ 0.00 & -2.30 $\pm$ 0.11 & -9.48 $\pm$ 0.30 & 0.19 $\pm$ 0.00 \\ 
fSVDKL-10& -0.02 $\pm$ 0.00 & 1.00 $\pm$ 0.00 & 0.78 $\pm$ 0.00 & -0.01 $\pm$ 0.00 & -1.14 $\pm$ 0.00 & -5.10 $\pm$ 0.01 & 0.14 $\pm$ 0.00 \\ 
pSVDKL-10& -0.00 $\pm$ 0.00 & 1.00 $\pm$ 0.00 & 0.79 $\pm$ 0.00 & -0.00 $\pm$ 0.00 & -1.13 $\pm$ 0.01 & -5.24 $\pm$ 0.05 & 0.15 $\pm$ 0.00 \\ 
\midrule 
NN-20 & - & 1.00 $\pm$ 0.00 & 0.79 $\pm$ 0.00 & -0.00 $\pm$ 0.00 & -2.06 $\pm$ 0.02 & -8.91 $\pm$ 0.14 & 0.18 $\pm$ 0.00 \\ 
SVDKL-20& -0.30 $\pm$ 0.20 & 0.91 $\pm$ 0.06 & 0.70 $\pm$ 0.02 & -0.26 $\pm$ 0.19 & -1.46 $\pm$ 0.16 & -4.72 $\pm$ 0.88 & 0.16 $\pm$ 0.04 \\ 
pNN-20 & - & 1.00 $\pm$ 0.00 & 0.79 $\pm$ 0.00 & -0.00 $\pm$ 0.00 & -2.24 $\pm$ 0.01 & -9.37 $\pm$ 0.08 & 0.19 $\pm$ 0.00 \\ 
fSVDKL-20& -0.02 $\pm$ 0.00 & 1.00 $\pm$ 0.00 & 0.79 $\pm$ 0.00 & -0.01 $\pm$ 0.00 & -1.14 $\pm$ 0.02 & -5.16 $\pm$ 0.10 & 0.13 $\pm$ 0.00 \\ 
pSVDKL-20& -0.00 $\pm$ 0.00 & 1.00 $\pm$ 0.00 & 0.79 $\pm$ 0.00 & -0.00 $\pm$ 0.00 & -1.09 $\pm$ 0.01 & -5.12 $\pm$ 0.05 & 0.15 $\pm$ 0.00 \\ 
\midrule 
NN-50 & - & 1.00 $\pm$ 0.00 & 0.79 $\pm$ 0.00 & -0.00 $\pm$ 0.00 & -2.18 $\pm$ 0.01 & -9.23 $\pm$ 0.04 & 0.19 $\pm$ 0.00 \\ 
SVDKL-50& -2.30 $\pm$ 0.00 & 0.10 $\pm$ 0.00 & 0.10 $\pm$ 0.00 & -2.30 $\pm$ 0.00 & -2.30 $\pm$ 0.00 & -2.30 $\pm$ 0.00 & 0.00 $\pm$ 0.00 \\ 
pNN-50 & - & 1.00 $\pm$ 0.00 & 0.79 $\pm$ 0.00 & -0.00 $\pm$ 0.00 & -2.38 $\pm$ 0.06 & -9.73 $\pm$ 0.15 & 0.19 $\pm$ 0.00 \\ 
fSVDKL-50& -0.02 $\pm$ 0.00 & 1.00 $\pm$ 0.00 & 0.79 $\pm$ 0.00 & -0.00 $\pm$ 0.00 & -1.22 $\pm$ 0.01 & -5.48 $\pm$ 0.05 & 0.14 $\pm$ 0.00 \\ 
pSVDKL-50& -0.00 $\pm$ 0.00 & 1.00 $\pm$ 0.00 & 0.79 $\pm$ 0.00 & -0.00 $\pm$ 0.00 & -1.11 $\pm$ 0.01 & -5.13 $\pm$ 0.04 & 0.15 $\pm$ 0.00 \\ 
\bottomrule
    \end{tabular}
\end{table}

\newpage
\section{Tabulated UCI Results}\label{app:uci_tables}
Here we tabulate the results for the UCI datasets.

\begin{table}[h]
    \centering
    \caption{Results for \textsc{Boston}. We report means plus or minus one standard error averaged over the splits.}\label{tab:boston}
    \begin{tabular}{rccccc}
\toprule 
& loss & train RMSE & test RMSE & train LL & test LL \\ 
\midrule 
SVGP & 1.66 $\pm$ 0.06 & 0.39 $\pm$ 0.01 & 0.37 $\pm$ 0.02 & -0.34 $\pm$ 0.01 & -0.33 $\pm$ 0.05 \\ 
fNN & 0.01 $\pm$ 0.00 & 0.02 $\pm$ 0.00 & 0.39 $\pm$ 0.03 & 2.28 $\pm$ 0.03 & -132.41 $\pm$ 22.39 \\ 
sNN & 0.01 $\pm$ 0.00 & 0.10 $\pm$ 0.00 & 0.34 $\pm$ 0.02 & 0.93 $\pm$ 0.02 & -5.61 $\pm$ 1.03 \\ 
DKL & -2.47 $\pm$ 0.00 & 0.00 $\pm$ 0.00 & 0.41 $\pm$ 0.02 & 2.72 $\pm$ 0.00 & -67.55 $\pm$ 3.97 \\ 
SVDKL & -0.47 $\pm$ 0.01 & 0.13 $\pm$ 0.00 & 0.35 $\pm$ 0.02 & 0.57 $\pm$ 0.01 & -1.12 $\pm$ 0.24 \\ 
      \bottomrule 
    \end{tabular}
\end{table}

\begin{table}[h]
    \centering
    \caption{Results for \textsc{Energy}.}\label{tab:energy}
    \begin{tabular}{rccccc}
\toprule 
& loss & train RMSE & test RMSE & train LL & test LL \\ 
\midrule 
SVGP & 0.07 $\pm$ 0.01 & 0.19 $\pm$ 0.00 & 0.20 $\pm$ 0.00 & 0.19 $\pm$ 0.01 & 0.15 $\pm$ 0.02 \\ 
fNN & 0.00 $\pm$ 0.00 & 0.02 $\pm$ 0.00 & 0.04 $\pm$ 0.00 & 2.55 $\pm$ 0.02 & -0.04 $\pm$ 0.38 \\ 
sNN & 0.00 $\pm$ 0.00 & 0.02 $\pm$ 0.00 & 0.05 $\pm$ 0.00 & 2.31 $\pm$ 0.02 & 0.62 $\pm$ 0.19 \\ 
DKL & -3.01 $\pm$ 0.02 & 0.01 $\pm$ 0.00 & 0.05 $\pm$ 0.00 & 3.15 $\pm$ 0.02 & -2.63 $\pm$ 0.49 \\ 
SVDKL & -1.21 $\pm$ 0.00 & 0.03 $\pm$ 0.00 & 0.04 $\pm$ 0.00 & 1.26 $\pm$ 0.00 & 1.22 $\pm$ 0.01 \\ 
      \bottomrule 
    \end{tabular}
\end{table}

\begin{table}[h]
    \centering
    \caption{Results for \textsc{Kin40K}.}\label{tab:kin40k}
    \begin{tabular}{rccccc}
\toprule 
& loss & train RMSE & test RMSE & train LL & test LL \\ 
\midrule 
SVGP & -0.14 $\pm$ 0.00 & 0.16 $\pm$ 0.00 & 0.17 $\pm$ 0.00 & 0.36 $\pm$ 0.00 & 0.33 $\pm$ 0.00 \\ 
fNN & 0.01 $\pm$ 0.00 & 0.03 $\pm$ 0.00 & 0.05 $\pm$ 0.00 & 2.18 $\pm$ 0.00 & 1.17 $\pm$ 0.02 \\ 
sNN & 0.01 $\pm$ 0.00 & 0.03 $\pm$ 0.00 & 0.05 $\pm$ 0.00 & 2.03 $\pm$ 0.00 & 1.51 $\pm$ 0.01 \\ 
VDKL & -1.41 $\pm$ 0.00 & 0.02 $\pm$ 0.00 & 0.05 $\pm$ 0.00 & 1.44 $\pm$ 0.00 & 1.33 $\pm$ 0.00 \\ 
SVDKL & -2.62 $\pm$ 0.00 & 0.01 $\pm$ 0.00 & 0.03 $\pm$ 0.00 & 2.68 $\pm$ 0.00 & 1.73 $\pm$ 0.02 \\ 
      \bottomrule 
    \end{tabular}
\end{table}

\begin{table}[h]
    \centering
    \caption{Results for \textsc{Power}.}\label{tab:power}
    \begin{tabular}{rccccc}
\toprule 
& loss & train RMSE & test RMSE & train LL & test LL \\ 
\midrule 
SVGP & -0.01 $\pm$ 0.00 & 0.23 $\pm$ 0.00 & 0.23 $\pm$ 0.00 & 0.06 $\pm$ 0.00 & 0.07 $\pm$ 0.01 \\ 
fNN & 0.04 $\pm$ 0.00 & 0.17 $\pm$ 0.00 & 0.21 $\pm$ 0.00 & 0.37 $\pm$ 0.00 & 0.11 $\pm$ 0.02 \\ 
sNN & 0.05 $\pm$ 0.00 & 0.21 $\pm$ 0.00 & 0.22 $\pm$ 0.00 & 0.14 $\pm$ 0.00 & 0.11 $\pm$ 0.01 \\ 
VDKL & -0.57 $\pm$ 0.00 & 0.13 $\pm$ 0.00 & 0.21 $\pm$ 0.00 & 0.62 $\pm$ 0.00 & -0.02 $\pm$ 0.02 \\ 
SVDKL & -0.25 $\pm$ 0.00 & 0.18 $\pm$ 0.00 & 0.21 $\pm$ 0.00 & 0.28 $\pm$ 0.00 & 0.16 $\pm$ 0.01 \\ 
      \bottomrule 
    \end{tabular}
\end{table}

\begin{table}[h]
    \centering
    \caption{Results for \textsc{Protein}.}\label{tab:protein}
    \begin{tabular}{rccccc}
\toprule 
& loss & train RMSE & test RMSE & train LL & test LL \\ 
\midrule 
SVGP & 1.06 $\pm$ 0.00 & 0.64 $\pm$ 0.00 & 0.66 $\pm$ 0.00 & -0.98 $\pm$ 0.00 & -1.00 $\pm$ 0.00 \\ 
fNN & 0.19 $\pm$ 0.00 & 0.39 $\pm$ 0.00 & 0.58 $\pm$ 0.00 & -0.46 $\pm$ 0.00 & -1.09 $\pm$ 0.01 \\ 
sNN & 0.17 $\pm$ 0.00 & 0.35 $\pm$ 0.00 & 0.55 $\pm$ 0.00 & -0.36 $\pm$ 0.00 & -1.14 $\pm$ 0.01 \\ 
VDKL & 0.32 $\pm$ 0.01 & 0.30 $\pm$ 0.00 & 0.59 $\pm$ 0.00 & -0.23 $\pm$ 0.01 & -1.86 $\pm$ 0.01 \\ 
SVDKL & 0.35 $\pm$ 0.00 & 0.31 $\pm$ 0.00 & 0.57 $\pm$ 0.00 & -0.26 $\pm$ 0.00 & -1.29 $\pm$ 0.01 \\ 

      \bottomrule 
    \end{tabular}
\end{table}

\end{document}